%% file: iclr2027_conference.tex
\documentclass{article} 
\usepackage{iclr2027_conference,times}

\input{packages}

\usepackage{hyperref}
\usepackage{url}

\title{A Banach-Space Theory of Markovian Halpern Iteration for Non-Expansive Maps}

\author{Ege C.~Kaya, Arda Fazla, M. Berk~Sahin \& Abolfazl Hashemi \\
Elmore Family School of Electrical and Computer Engineering\\
Purdue University\\
West Lafayette, IN 47907, USA \\
\texttt{\{kayae,afazla,sahinm,abolfazl\}@purdue.edu}
}

\iclrfinalcopy 
\begin{document}

\maketitle

\begin{abstract}
We study stochastic approximation of fixed points of a non-expansive operator $T$ when the oracle samples originate from a continuing Markovian trajectory. A direct block-minibatch implementation of Halpern iteration attains an expected last-iterate residual of order $O(\log N/N)$, but accrues a substantive complexity of $\tilde O(\epsilon^{-5})$ Markovian samples. We therefore introduce a variance-reduced Markovian PAGE-Halpern method whose refresh and same-state difference blocks are analyzed through the Poisson equation. In Hilbert spaces, the cocoercivity of $I-T$ results in an $O(\epsilon^{-3})$ sample complexity. Our main result extends this construction to a general finite-dimensional Banach space. A displacement-level Halpern bound replaces the Hilbert-space potential and yields $\tilde O(\epsilon^{-3})$ sample complexity. We also establish a high-probability guarantee with the same leading accuracy dependence by measuring the estimator in an auxiliary smooth norm.

\end{abstract}

\section{Introduction}

Many stochastic approximation (SA) algorithms can be expressed as fixed-point iterations for an operator that can be only observed through noisy samples \citep{robbins1951stochastic, kushner2003stochastic, borkar2008stochastic}. Contractive operators, in particular, already admit a well-developed finite-iteration theory because every exact update is guaranteed to decrease the distance to the unique fixed point \citep{srikant2019finite, qu2020finite, chen2020finite}. Non-expansive operators present different challenges. Namely, the fixed point is not necessarily unique, there is generally no strict contraction drift, and a sampling error can persist without being geometrically attenuated.

Halpern iteration \citep{halpern1967fixed} provides a natural deterministic method in this regime. Given an anchor $x_0$ and a non-expansive operator $T$, the canonical Halpern recursion is given by
\begin{equation}
x_n = \frac{1}{n+1} x_0 + \frac{n}{n+1}T x_{n-1}.
\end{equation}
The method has an $O(1/N)$ last-iterate fixed-point residual in Hilbert spaces, and the initial choice of the anchor can be used to select a specific fixed point \citep{wittmann1992approximation, lieder2020convergence}. These properties make it attractive for non-expansive stochastic fixed-point problems, including average-reward reinforcement learning (RL) \citep{tsitsiklis1999average, abounadi2001learning}.

The present paper considers the setting where the samples are generated by one continuing Markov chain. The oracle error need not have zero conditional mean at the start of a Markovian block, and successive blocks remain dependent through the state at their common boundary. We are therefore in need of an analysis that preserves the online trajectory and accounts explicitly for both the martingale component of the Markovian error and its conditional bias at each block boundary. Furthermore, ordinary minibatching estimates every operator value independently, resulting in an overall complexity of $\tilde O(\epsilon^{-5})$. To improve this rate, variance reduction methods such as PAGE \citep{li2021page} can be used when a single Markov state can be used to evaluate the oracle at two query points. The estimator can then track the change in the operator rather than reconstructing its value at every iteration.

In Hilbert spaces, this idea can be coupled with the cocoercivity of $F = I - T$ \citep{cai2022stochastic}. The same argument, however, is unavailable in the sup and block-sup norms that arise frequently in RL \citep{abounadi2001learning, bellemare2023distributional}. Our approach, in this case, is to move the recursive estimator back from $F$ to $T$ and control the increments of the Halpern displacement directly. This gives a Banach-space argument in the norm in which the operator is known to be non-expansive, extending the theory to geometries where an inner-product structure is not available.

\noindent\textbf{We summarize our contributions as follows.}
\begin{itemize}[leftmargin=1.4em,itemsep=1pt,topsep=2pt]
\item We establish perturbation- and displacement-level bounds for inexact Halpern iteration. The former yields a Markovian block baseline with $\tilde O(\epsilon^{-5})$ sample complexity, while the latter enables recursive variance reduction.

\item We develop PAGE-Halpern along one uninterrupted Markov trajectory. Using a Poisson-equation decomposition, we obtain $O(\epsilon^{-3})$ sample complexity in Hilbert space and $\tilde O(\epsilon^{-3})$ in finite-dimensional Banach spaces, including sup and block-sup geometries.

\item We prove high-probability Banach-space guarantees using an auxiliary smooth norm. Smooth approximations of the sup norm incur only logarithmic dimension dependence \citep{juditsky2023largedeviationsvectorvaluedmartingales,chen2020finite}.
\end{itemize}

\noindent\textbf{Related work.}
The closest results study stochastic Halpern iteration under i.i.d.\ sampling \citep{cai2022stochastic,bravo2026stochastic}, Markovian SA for non-expansive maps through Krasnosel'ski\u{i}--Mann iteration \citep{blaser2026asymptotic}, and variance reduction along a single Markov trajectory \citep{sun2026learning}. Our results combine Halpern anchoring, recursive same-state differences, and Poisson-equation control of one continuing trajectory, and extend the analysis to general finite-dimensional Banach norms. \cref{app:related} provides a fuller comparison with these and related lines of work.

\section{Setup and inexact Halpern bounds}\label{sec:prelim}

Let $(\bR^d, \norm{\cdot})$ be a finite-dimensional normed space, write $\norm{\cdot}_2$ for the Euclidean norm and let
\begin{equation}
\mu_{\norm{\cdot}} := \sup_{z \ne 0}\frac{\norm{z}}{\norm{z}_2} < \infty.
\end{equation}
Let $(Y_t)_{t \ge 0}$ be a Markov chain on a finite state space $\cY$, with transition matrix $P$ and stationary distribution $\pi$. Given a measurable map $H : \R^d \times \cY \to \R^d$, define the mean operator
\begin{equation}
Tx := \sum_{y \in \cY} \pi(y) H(x, y).
\end{equation}
We seek an approximate fixed point of $T$, under the following assumptions.

\begin{assumption}[Non-expansive operator with fixed points]\label{ass:nonexpansive}
The operator $T$ is non-expansive, meaning $\norm{Tx - Tz} \le \norm{x - z}$ for all $x, z \in \R^d$, and $\Fix(T) \ne \emptyset$.
\end{assumption}
\begin{assumption}[Finite ergodic Markov chain]\label{ass:ergodic}
The finite chain $(Y_t)_{t \ge 0}$ is irreducible and aperiodic. Thus, there are constants $C_{\mathrm{mix}} \ge 1$ and $\rho \in (0, 1)$ such that $\sup_{y \in \cY} \norm{P^t(y, \cdot) - \pi}_{\mathrm{TV}} \le C_{\mathrm{mix}} \rho^t$ for every $t \ge 0$.
\end{assumption}
\begin{assumption}[Finite oracle variance]\label{ass:oracle-variance}
For every $x \in \R^d$, $\sum_{y \in \cY} \pi(y) \norm{H(x, y) - Tx}^2_2 < \infty$.
\end{assumption}

\begin{assumption}[Iterate stability]\label{ass:oracle-stable}
For $q \in \{1, 2\}$, there exist constants $R_q < \infty$ such that $\sup_{n \ge 0} \E \bigl[\norm{x_n - x_0}^q \bigr] \le R^q_q$ for the generated iterates of the SA algorithm under consideration.
\end{assumption}

The first assumption specifies the fixed-point geometry, and the second describes the sampling process. The remaining two control the random region visited by the iterates and ensure that the oracle has a finite second moment at every fixed query point. All four are standard SA conditions. Their relation to the assumptions used in stochastic Halpern and Markovian SA literature is discussed in \cref{app:related}.

We present two ways to control an inexact Halpern recursion. The first is a standard Markovian block method with adaptive batch size scheduling, which builds a fresh operator estimate at each step and bounds the magnitude of its error. \cref{sec:baseline} analyzes this naive method and obtains an $O(\log N/N)$ expected last-iterate residual with $\tilde O(\epsilon^{-5})$ samples. The high cost of repeatedly forming accurate block estimates motivates the use of a more sophisticated, variance-reduced estimator, which updates earlier estimates using the same Markovian states at two successive query points. The analysis of this second method needs a residual bound in terms of changes in the estimation error. For an arbitrary error sequence $(U_n)_{n\ge1}$, consider the inexact Halpern update
\begin{equation}\label{eq:recursion-with-error}
x_n = (1-\beta_n)x_0 + \beta_n(Tx_{n-1}+U_n).
\end{equation}
The next theorem bounds the last-iterate residual using the increments of the scaled errors $\beta_nU_n$, together with the error at the final step.
\begin{theorem}[Displacement-level bound]\label{thm:displacement-level}
Under Assumptions~\ref{ass:nonexpansive} and~\ref{ass:oracle-stable}, let $D_0 := \dist(x_0,\Fix(T))$ and $\kappa_1 := R_1+2D_0$. Let $(x_n)$ be generated by \eqref{eq:recursion-with-error} with $\beta_n=n/(n+1)$. Set $U_0:=0$ and $\Delta_n:=\beta_nU_n-\beta_{n-1}U_{n-1}$. If $\E[\norm{\Delta_n}^2]\le q_n$ and $\E[\norm{U_N}]\le r$, then
\begin{equation}
\E\bigl[\norm{x_N-Tx_N}\bigr] \le \frac{\kappa_1(1+\log(N+1))}{N+1} + \frac{1}{N+1}\sum_{n=1}^N n\sqrt{q_n}+r.
\end{equation}
\end{theorem}

The bound is suited to recursive estimation, where changes in the error are easier to control than the error at every iteration. We use it for the Banach-space PAGE analysis in \cref{sec:banach}. The proof is provided in \cref{app:inexact-details}, and a pathwise version is given in \cref{app:hp-details}.

\section{Markovian PAGE-Halpern in Hilbert spaces}\label{sec:hilbert}

The naive block baseline method of \cref{sec:baseline} requires $\tilde O(\epsilon^{-5})$ samples. In the i.i.d.\ setting, \citet{bravo2026stochastic} obtain an $O(\epsilon^{-5})$ upper bound and an $\Omega(\epsilon^{-3})$ lower bound for linear-span algorithms with bounded-variance single-point oracle access. The gap between these exponents motivates the recursive construction in this section. Since i.i.d.\ sampling is a special case of Markovian sampling, the lower bound remains a useful benchmark, although our PAGE method uses same-state evaluations at two query points. The apparent reason for the inefficency of the baseline method is that it treats the oracle errors at different query points as unrelated. Under only pointwise finite variance and affine variance growth, this is essentially all we can use, since a sample at $x$ gives no information about the oracle error at a nearby point $z$. Furthermore, the resulting residual bound depends on the weighted sum of perturbations $\sum_{n=1}^N n \E[\norm{U_n}]$. Keeping this sum $O(\log N)$ requires the $n$th block error to be of order $n^{-2}$. Since averaging reduces standard deviation only as $k_n^{-1/2}$, the block sizes grow as $k_n=n^4$, and hence to the $\tilde O(\epsilon^{-5})$ sample complexity.

Variance reduction can make more efficient use of observed samples, albeit with the prerequisite of a stronger oracle model. The relevant additional regularity is that, using the same Markov state, two evaluations at nearby query points have a small mean-square difference after subtracting their stationary means. Under this condition, a recursive estimator can be updated from same-sample differences instead of being rebuilt from scratch at every iteration. The stochastic cost is then governed by the displacement of the Halpern iterates rather than the full variance of a fresh estimate. Concretely, the estimator alternates between two kinds of blocks. A refresh block averages $\widehat F(x, Y)$ at a single query point to estimate $F(x)$. A recursive block uses the same Markov states at two query points and averages $\widehat F(x, Y) - \widehat F(z, Y)$ to estimate the difference $F(x) - F(z)$.

In a Hilbert space, there is a direct way to combine a PAGE-like variance-reduced estimator with a residual bound. If $T$ is non-expansive, then $F := I - T$ is monotone and $1/2$-cocoercive, since
\begin{equation}
\norm{Fx - Fz}^2_2 = \norm{(x-z) - (Tx - Tz)}^2_2 \le 2 \langle Fx - Fz, x-z\rangle.
\end{equation}
Thus, controlling the residual $\norm{x - Tx}_2$ is equivalent to controlling $\norm{F(x)}_2$ \citep{cai2022stochastic}. 

\subsection{Markovian PAGE oracle}

The analysis of this section uses the stochastic residual oracle $\widehat F(x, y) := x - H(x, y)$, so that
\begin{equation}
F(x) = \E_\pi\bigl[\widehat F(x, Y)\bigr].
\end{equation}
The additional structure needed for recursive variance reduction is stated in the following condition.
\begin{condition}[PAGE-compatible Markovian oracle]\label{cond:page}
The stochastic residual oracle permits multi-point access, i.e., after observing a state $Y$, the algorithm may evaluate both $\widehat F(x, Y)$ and $\widehat F(z, Y)$. Furthermore, the oracle has stationary second moments that satisfy
\begin{equation}
\sum_{y \in \cY} \pi(y) \norm{\widehat F(x, y) - F(x)}^2_2 \le \sigma^2_F,
\end{equation}
and same-sample difference second moments that satisfy
\begin{equation}
\sum_{y \in \cY} \pi(y) \norm{\widehat F(x, y) - \widehat F(z, y) - (F(x) - F(z))}^2_2 \le L^2_F \norm{x -z}^2_2.
\end{equation}
\end{condition}

\subsection{The Markovian PAGE-Halpern scheme}

We now describe the Markovian, Hilbert-space PAGE-Halpern recursion. The scheme observes one continuing trajectory of the Markov chain $(Y_t)_{t \ge 0}$ with natural filtration $(\cF_t)$ enlarged by past randomization. Each block starts at a stopping time $\tau$, and its length $S$ is an integer-valued $\cF_\tau$-measurable random variable, chosen before any state within the block is observed. Query points are also $\cF_\tau$-measurable. For a block starting after $\tau$, we define the refresh and difference block averages
\begin{equation}
R(x;\tau, S) := \frac{1}{S}\sum_{i=1}^S \widehat F(x, Y_{\tau+i}) \quad \text{and} \quad
D(x, z; \tau, S) = \frac{1}{S}\sum_{i=1}^S \bigl(\widehat F(x, Y_{\tau + i}) - \widehat F(z, Y_{\tau + i})\bigr).
\end{equation}

Letting $x_0 \in \R^d$, we choose $L^* \ge \max\{2, L_F\}$, and set $p_n = \min\{1, {2}/({n+1})\}$ and $\beta_n = n/(n+1)$. The choice of $L^*$ upper bounds both the cocoercivity scale of $F = I -T$ and the same-sample Lipschitz scale in \cref{cond:page}. We initialize the chain time at $\tau = 0$ and set $v_0 = R(x_0; \tau, S_{1, 0})$. Given $v_{n-1}$, we compute
\begin{equation}\label{eq:page-halpern-1}
x_n = (1- \beta_n)x_0 + \beta_n\left(x_{n-1} - \frac{1}{L^*}v_{n-1}\right),
\end{equation}
and then update
\begin{equation}\label{eq:page-halpern-2}
v_n = \begin{cases}
R(x_n ; \tau, S_{1,n}), &\text{with probability } p_n, \\
v_{n-1} + D(x_n, x_{n-1}; \tau, S_{2, n}), &\text{with probability } 1-p_n.
\end{cases}
\end{equation}
Thus, $v_n$ is a recursive PAGE estimator of $F(x_n)$, updated either by a refresh block or by a difference block. The PAGE coin at iteration $n$ is tossed before choosing the corresponding refresh or difference block length. The random quantities $S_{1, n}$ and $S_{2, n}$ may depend on the past, including the PAGE coin toss and the already computed displacement $\norm{x_n - x_{n-1}}$, but not on states within that block. Each estimator used in the update is therefore measurable before the start of the next Markovian block.

\subsection{PAGE error and sample complexity}

The chain state at the beginning of each block is inherited from all preceding PAGE updates. Consequently, the samples within the block are not independent of the past and need not begin in stationarity. We control this dependence through the stopping-time Poisson decomposition proved in \cref{app:hilbert-details}. The main result states that for a function $g$ with zero stationary mean, if the length $S$ is chosen at a stopping time $\tau$, then the block average has conditional bias at most $C_{\mathrm B}\norm{g}_{2,\pi}/S$ and conditional second moment at most $C_{\mathrm V}\norm{g}_{2,\pi}^2/S$. The proof uses the Poisson equation to decompose the block sum into a martingale sum and a boundary term involving only the first and last states. The martingale sum gives the $1/S$ second-moment decay, while the boundary term gives the $1/S$ bias. \cref{lem:page-blocks} applies this result to the two block types used by PAGE, giving a conditional bias bound and a conditional second-moment bound for both the refresh and same-state difference blocks. The following lemma uses these estimates to choose the block sizes for a prescribed estimator accuracy.

\begin{lemma}[Hilbert PAGE schedule]\label{lem:page-schedule} Fix a target estimator accuracy $\delta>0$ and choose
\begin{equation}\label{eq:refresh-block-size}
S_{1,0} \ge \frac{C_{\mathrm V}\sigma_F^2}{\delta^2}, \; S_{1,n} \ge \frac{C_{\mathrm V}\sigma_F^2}{p_n\delta^2}, \; S_{2,n} \ge \max\left\{\frac{C_{\mathrm V}L_F^2\norm{x_n-x_{n-1}}_2^2}{p_n^2\delta^2}, \frac{C_{\mathrm B}L_F\norm{x_n-x_{n-1}}_2}{p_n^{3/2}\delta}, 1 \right\}.
\end{equation}
Then the PAGE-Halpern recursion \eqref{eq:page-halpern-1}--\eqref{eq:page-halpern-2} satisfies
\begin{equation}
\E\bigl[\norm{v_n-F(x_n)}_2^2\bigr] \le \frac{100\delta^2}{n+1}, \qquad n \ge 0.
\end{equation}
\end{lemma}

The Hilbert-space residual potential absorbs this error while simultaneously controlling the displacements that determine the cost of the difference blocks. The required cocoercive estimate is due to \citet{cai2022stochastic}. We state it in our notation below and give the supporting details in \cref{app:hilbert-details}.

\begin{proposition}[PAGE-Halpern residual in Hilbert spaces]
\label{prop:cocoercive-page}
Let $F:\R^d \to \R^d$ be monotone and $1/L^*$-cocoercive with $F^{-1}(0) \ne \emptyset$.  Let $R := \dist(x_0, F^{-1}(0))$, and generate $(x_n)$ by \eqref{eq:page-halpern-1}--\eqref{eq:page-halpern-2}. If
\begin{equation}
\E\bigl[\norm{v_n-F(x_n)}_2^2\bigr] \le \frac{\delta^2}{n+1}, \qquad n\ge 0,
\end{equation}
then there exist universal constants $C_0, C_1$ such that
\begin{equation}\label{eq:stoch-halpern-potential}
\E\norm{F(x_N)}_2 \leq \frac{C_0 L^* R}{N}+C_1\delta.
\end{equation}
Moreover, there exist universal constants $C_2, c_2 > 0$ such that, if $p_n = 2/(n+1)$ and $\delta N \le c_2 L^*R$,
\begin{equation}\label{eq:halpern-displacement}
\sum_{n=1}^N \frac{\E\bigl[\norm{x_n-x_{n-1}}_2^2\bigr]}{p_n^2} \le C_2 R^2 N.
\end{equation}
\end{proposition}

Combining \cref{prop:cocoercive-page} with the Markovian PAGE schedule provides us with a complexity result. The proposition controls the residual and the displacements that determine the cost of recursive blocks and the Poisson decomposition makes these blocks valid along one continuing Markov trajectory.

\begin{theorem}[Markovian PAGE-Halpern complexity in Hilbert spaces]\label{thm:page-halpern}
Suppose Assumption~\ref{ass:nonexpansive} holds on a Hilbert space, and let $F := I - T$. Suppose further that Assumption~\ref{ass:ergodic} and Condition~\ref{cond:page} hold. Consider the PAGE-Halpern recursion \eqref{eq:page-halpern-1}--\eqref{eq:page-halpern-2}, and let $R := \dist(x_0,\Fix(T))$. Choose $L^*\ge\max\{2,L_F\}$ and use the block-size schedules of Lemma~\ref{lem:page-schedule}. Let $C_0,C_1$ be the constants in \cref{prop:cocoercive-page}. If $R=0$ or $\epsilon\ge 2R$, the anchor $x_0$ already has residual at most $\epsilon$. Otherwise, 
\begin{equation}
\delta = \frac{\epsilon}{20C_1}, \qquad N = \left\lceil \frac{2C_0L^* R}{\epsilon}\right\rceil
\end{equation}
ensures
\begin{equation}
\E\bigl[\norm{x_N-Tx_N}_2\bigr] = \E\bigl[\norm{F(x_N)}_2\bigr] \le \epsilon.
\end{equation}
Moreover, the expected number of Markovian samples used by the continuing trajectory is
\begin{equation}\label{eq:page-complexity}
O\left(\frac{C_{\mathrm V}\bigl(\sigma_F^2 L^* R + (L^* R)^3\bigr)}{\epsilon^3} + \frac{C_{\mathrm B} (L^*R)^{5/2}}{\epsilon^{5/2}}\right).
\end{equation}
In particular, the leading dependence on the target accuracy $\epsilon$ is $O(\epsilon^{-3})$.
\end{theorem}

The theorem improves the $\tilde O(\epsilon^{-5})$ block-baseline complexity to $O(\epsilon^{-3})$, with its leading accuracy exponent agrees with the i.i.d.\ lower bound, while the Markovian boundary contribution is the lower-order $O(\epsilon^{-5/2})$ term. In particular, temporal dependence changes the chain-dependent constants but introduces no additional logarithmic factor in the leading Hilbert-space rate.

\section{Markovian PAGE-Halpern in Banach spaces}\label{sec:banach}

The Hilbert-space result of the previous section depends on the fortunate geometric identity that if $T$ is non-expansive in a Hilbert norm, then $F := I - T$ is cocoercive. This identity lets us convert residual control into a monotone-inclusion estimate, and provides PAGE with a quadratic potential into which its estimator error can be inserted. However, norms that arise in many common applications, such as the sup norms, span seminorms and block-sup norms of RL and distributional RL, do not have an inner-product geometry. The aim of this section is to establish similar results in general Banach spaces. Rather than estimating the residual operator $F = I - T$, we estimate the fixed-point operator $T$ itself, and use the resulting estimator along with the displacement-level Halpern bound of Theorem~\ref{thm:displacement-level}. The anchoring in the Halpern iteration plays the role of the cocoercivity in the Hilbert analysis in converting control of the estimator increments into control of the fixed-point residual.

\subsection{Markovian PAGE oracle}

We will now use the recursive estimator to estimate $Tx$ rather than $F(x)$. Therefore, we need an analogue of Condition~\ref{cond:page} which is imposed instead on the original stochastic operator $H$.

\begin{condition}[PAGE-compatible Markovian oracle in Banach spaces]\label{cond:page-banach}
The stochastic operator permits multi-point access, i.e., after observing a state $Y$, the algorithm may evaluate both $H(x, Y)$ and $H(z, Y)$. Furthermore, every query point $x$ generated up to the target horizon satisfies
\begin{equation}
\sum_{y \in \cY} \pi(y) \norm{H(x, y) - Tx}^2_2 \le \sigma^2_T,
\end{equation}
and every query-point pair $x, z$ generated up to the target horizon satisfies
\begin{equation}
\sum_{y \in \cY} \pi(y) \norm{H(x, y) - H(z, y) - (Tx - Tz)}^2_2 \le L^2_T \norm{x - z}^2.
\end{equation}
\end{condition}
The input displacement in the second bound is measured in $\norm{\cdot}$, the non-expansiveness norm of $T$.

\subsection{The Markovian PAGE-Halpern scheme}

The scheme observes the same continuing Markovian trajectory as in the previous section. For a block starting after the current chain time $\tau$, let
\begin{equation}
R_T(x; \tau, S) := \frac{1}{S} \sum_{i=1}^S H(x, Y_{\tau + i}) \quad \text{and} \quad
D_T(x, z; \tau, S) := \frac{1}{S} \sum_{i=1}^S \bigl( H(x, Y_{\tau + i}) - H(z, Y_{\tau + i})\bigr).
\end{equation}
Letting $x_0 \in \R^d$, we set $p_n = \min\{1, 2/(n+1)\}$ and $\beta_n = n/(n+1)$. We initialize the chain time at $\tau = 0$ and set $v_0 = R_T(x_0; \tau, S_{1, 0})$. Given $v_{n-1}$, we compute
\begin{equation}\label{eq:page-halpern-banach-1}
x_n = (1 - \beta_n) x_0 + \beta_n v_{n-1},
\end{equation}
and then update
\begin{equation}\label{eq:page-halpern-banach-2}
v_n = \begin{cases}
R_T(x_n; \tau, S_{1, n}), & \text{with probability } p_n, \\
v_{n-1} + D_T(x_n, x_{n-1}; \tau, S_{2, n}), & \text{with probability } 1-p_n.
\end{cases}
\end{equation}
Evidently, $v_n$ is a recursive PAGE estimator directly of $T x_n$. The same Poisson decomposition argument as in the Hilbert analysis controls its refresh and difference blocks. The main difference is how estimator accuracy is converted into residual control. In the absence of cocoercivity, we use \cref{thm:displacement-level}, which bounds the residual through the increments of the scaled estimator error. The schedule in the next lemma makes these increments decay at the required $1/n$ rate. Since the difference-block lengths depend on $\norm{x_n - x_{n-1}}$, their total sample cost can then be controlled by the Halpern displacement bounds. Together, these estimates replace the Hilbert-space cocoercive potential and yield a variance-reduced guarantee in the original Banach norm.

\begin{lemma}[Banach PAGE schedule]\label{lem:page-schedule-banach}
Fix a target estimator accuracy $\delta >0$. Choose 
\begin{equation}\label{eq:refresh-block-size-banach}
S_{1,0} \ge \frac{C_{\mathrm V}\sigma_T^2}{\delta^2}, \; S_{1,n} \ge \frac{C_{\mathrm V}\sigma_T^2}{p_n\delta^2}, \; S_{2,n} \ge \max\left\{\frac{C_{\mathrm V}L_T^2\norm{x_n-x_{n-1}}^2}{p_n^2\delta^2}, \frac{C_{\mathrm B}L_T\norm{x_n-x_{n-1}}}{p_n^{3/2}\delta}, 1 \right\}.
\end{equation}
Then the PAGE-Halpern recursion \eqref{eq:page-halpern-banach-1}--\eqref{eq:page-halpern-banach-2} satisfies
\begin{equation}\label{eq:banach-estimator-level}
\E\bigl[\norm{v_n-Tx_n}_2^2\bigr] \le \frac{100\delta^2}{n+1}, \qquad n \ge 0.
\end{equation}
\end{lemma}

The displacement-level reduction of \cref{thm:displacement-level} converts the estimator control into a residual guarantee in the non-expansiveness norm. The next theorem is the central geometric extension of the Hilbert result, retaining the sample efficiency of recursive estimation without using cocoercivity.

\begin{theorem}[Markovian PAGE-Halpern in Banach spaces]\label{thm:page-halpern-banach}
Suppose \cref{ass:nonexpansive,ass:ergodic,ass:oracle-stable} hold in the non-expansiveness norm $\norm{\cdot}$. Suppose further that \cref{cond:page-banach} holds. Let $\kappa := R_2 + 2\dist(x_0,\Fix(T))$. Consider the PAGE-Halpern recursion \eqref{eq:page-halpern-banach-1}--\eqref{eq:page-halpern-banach-2} and use the block-size schedules of \cref{lem:page-schedule-banach}. There exists a universal constant $C$ such that for every $N \ge 1$,
\begin{equation}\label{eq:banach-page-residual}
\E\bigl[\norm{x_N - T x_N}\bigr] \le \frac{\kappa (1 + \log(N+1))}{N+1} + C \mu_{\norm{\cdot}}\delta.
\end{equation}
Consequently, choosing
\begin{equation}
\delta = \frac{\epsilon}{2 C \mu_{\norm{\cdot}}}, \qquad N = \Theta\left(\frac{\kappa}{\epsilon} \log \frac{\kappa}{\epsilon}\right)
\end{equation}
ensures 
\begin{equation}
\E\bigl[\norm{x_N - Tx_N}\bigr] \le \epsilon.
\end{equation}
Moreover, the expected number of Markovian samples used by the continuing trajectory is
\begin{equation}\label{eq:banach-page-complexity}
\tilde O\left( \frac{C_V\bigl(\mu^2_{\norm{\cdot}}\sigma^2_T \kappa + (1 + \mu^2_{\norm{\cdot}})L^2_T \kappa^3\bigr)}{\epsilon^3} + \frac{\mu_{\norm{\cdot}}C_B L_T \kappa^{5/2}}{\epsilon^{5/2}}\right).
\end{equation}
In particular, the leading dependence on the target accuracy is $\tilde O(\epsilon^{-3})$.
\end{theorem}

The first term in \eqref{eq:banach-page-residual} is the deterministic Halpern contribution. Its harmonic factor introduces a logarithmic overhead in $N$, so the Banach-space sample complexity is $\tilde O(\epsilon^{-3})$ rather than $O(\epsilon^{-3})$. The leading polynomial dependence on the target accuracy remains $\epsilon^{-3}$, as in the Hilbert result. The factor $\mu_{\norm{\cdot}}$ is incurred only when Euclidean estimator bounds are converted to the working norm. In particular, $\mu_{\norm{\cdot}}\le1$ for sup and block-sup norms, so these geometries incur no dimension-dependent norm-comparison loss. Further details are provided in \cref{app:banach-details}.

\section{High-probability guarantees}\label{sec:hp}

The results of the previous sections concern expected residual bounds. In this section we provide the corresponding high-probability statement in Banach spaces. The general structure of the analysis remains similar to Section~\ref{sec:banach}, the new ingredient being estimator concentration. Since martingale concentration in a general norm requires smoothness of said norm, our approach will be to measure the recursive estimator in an auxiliary smooth norm $\norm{\cdot}_s$, and to compare it back to $\norm{\cdot}$ at the end.

\subsection{Smooth high-probability oracle}

\begin{condition}[Smooth high-probability Markovian oracle]\label{cond:hp-page}
There is a norm $\norm{\cdot}_{\mathrm{s}}$ on $\R^d$ and constants $\alpha_{\mathrm{s}},\kappa_{\mathrm{s}} \ge 1$ such that
\begin{equation}
\norm{z} \le \alpha_{\mathrm{s}}\norm{z}_{\mathrm{s}}, \qquad z \in \R^d,
\end{equation}
and $(\R^d,\norm{\cdot}_{\mathrm{s}})$ is $\kappa_{\mathrm{s}}$-smooth, meaning $\norm{z + h}_{\mathrm{s}}^2 \le \norm{z}_{\mathrm{s}}^2 + \langle \nabla \norm{z}_{\mathrm{s}}^2, h\rangle + \kappa_{\mathrm{s}}\norm{h}_{\mathrm{s}}^2$ for all $z, h \in \R^d$.
The oracle permits Markovian multi-point access. Moreover, every generated query point $x$ satisfies
\begin{equation}
\sum_{y \in \cY} \pi(y)\norm{H(x,y) - Tx}_{\mathrm{s}}^2 \le \sigma_{\mathrm{s}}^2,
\end{equation}
and every generated query-point pair $x,z$ satisfies
\begin{equation}\label{eq:hp-page-diff}
\sum_{y \in \cY} \pi(y) \norm{H(x,y) - H(z,y) - (Tx - Tz)}_{\mathrm{s}}^2 \le L_{\mathrm{s}}^2 \norm{x-z}^2.
\end{equation}
\end{condition}

When the working norm is Euclidean, one may take $\norm{\cdot}_{\mathrm{s}} = \norm{\cdot}_2$, $\alpha_{\mathrm{s}} = 1$, and $\kappa_{\mathrm{s}} = 1$. When the working norm is non-smooth, the estimator norm can be smoothed without changing the fixed-point geometry. For example, if $\norm{\cdot} = \norm{\cdot}_\infty$, take $\norm{\cdot}_{\mathrm{s}} = \norm{\cdot}_q$ with $q = \max\{2, \lceil\log(\max\{e,d\})\rceil\}$. Then
\begin{equation}
\norm{z}_\infty \le \norm{z}_q \le d^{1/q} \norm{z}_\infty \le e\norm{z}_\infty,
\end{equation}
and $\kappa_{\mathrm{s}} = q - 1 = O(\log d)$. The same construction extends easily to block-sup norms.

The concentration certificate must be combined with a pathwise analogue of the stability bound used in expectation. We therefore impose the following finite-horizon condition.

\begin{condition}[Pathwise stability]\label{cond:hp-stability}
For the generated iterates up to the target horizon $N$, there exists a deterministic constant $K < \infty$ such that
\begin{equation}
\sup_{0 \le n \le N}\norm{Tx_n - x_0} \le K \quad \text{a.s.}
\end{equation}
\end{condition}

We now revisit the Banach-space PAGE-Halpern recursion of \eqref{eq:page-halpern-banach-1}--\eqref{eq:page-halpern-banach-2}. Let $C_{\mathrm{hp}}$ denote the chain-dependent constant in the high-probability block estimate of \cref{lem:hp-page-blocks} and define
\begin{equation}
\ell_N := 1 + \log(N+1), \qquad \Lambda_{N, \eta} := \kappa_s + \log(32(N+1)/\eta),
\end{equation}
and set $p_0 := 1$.

\begin{lemma}[High-probability PAGE schedule]\label{lem:hp-page-schedule}
Let $N \ge 1$, $\eta \in (0, 1)$, and $\delta > 0$. Choose
\begin{equation}\label{eq:hp-refresh-size}
S_{1,0} \ge C_{\mathrm{hp}}\left(\frac{\Lambda_{N,\eta}\sigma_{\mathrm{s}}^2}{\delta^2} + \frac{\sigma_{\mathrm{s}}}{\delta}\right), \quad S_{1,n} \ge C_{\mathrm{hp}}\left(\frac{\Lambda_{N,\eta}\sigma_{\mathrm{s}}^2}{p_n\delta^2} + \frac{\sigma_{\mathrm{s}}}{\sqrt{p_n}\delta}\right),
\end{equation}
and
\begin{equation}\label{eq:hp-difference-size}
S_{2,n} \ge C_{\mathrm{hp}}\left(\frac{\Lambda_{N,\eta}L_{\mathrm{s}}^2\norm{x_n - x_{n-1}}^2\ell_N^2}{p_n^2\delta^2} + \frac{L_{\mathrm{s}}\norm{x_n - x_{n-1}}\ell_N}{p_n\delta} + 1\right).
\end{equation}
Let $e_n := v_n - Tx_n$, with the convention $e_{-1} := 0$, and let $\cR_n$ denote the event that iteration $n$ uses a refresh block. Then, with probability at least $1 - \eta$,
\begin{equation}\label{eq:hp-estimator-uniform}
\max_{0 \le n \le N}\norm{e_n}_{\mathrm{s}} \le \delta, \qquad \sum_{n=1}^N n\norm{\beta_n e_{n-1} - \beta_{n-1}e_{n-2}}_{\mathrm{s}} \le C\delta N(1 + \log(1/\eta)),
\end{equation}
and
\begin{equation}\label{eq:hp-refresh-count}
\sum_{n=1}^N \frac{\mathbf 1_{\cR_n}}{p_n} \le C N(1 + \log(1/\eta)).
\end{equation}
\end{lemma}

The schedule controls the estimator error and its scaled increments simultaneously over all iterations up to $N$. Combining these bounds with pathwise stability and the displacement-level Halpern argument gives the following high-probability guarantee for the last iterate.

\begin{theorem}[High-probability Markovian PAGE-Halpern]\label{thm:hp-page}
Suppose \cref{ass:nonexpansive,ass:ergodic,cond:hp-page,cond:hp-stability} hold. Run the Banach PAGE-Halpern recursion \eqref{eq:page-halpern-banach-1}--\eqref{eq:page-halpern-banach-2} with the block-size schedules of \cref{lem:hp-page-schedule}. Then, with probability at least $1 - \eta$,
\begin{equation}\label{eq:hp-residual-bound}
\norm{x_N - Tx_N} \le \frac{K(1 + \log(N + 1))}{N + 1} + C\alpha_{\mathrm{s}}\delta(1 + \log(1/\eta)).
\end{equation}
Consequently, the choices
\begin{equation}
N = \Theta\left(\frac{K}{\epsilon}\log\frac{K}{\epsilon}\right), \qquad \delta = \Theta\left(\frac{\epsilon}{\alpha_{\mathrm{s}}(1 + \log(1/\eta))}\right)
\end{equation}
ensure $\norm{x_N - Tx_N} \le \epsilon$ with probability at least $1 - \eta$. On the same event, the number of Markovian samples used through time $N$ is
\begin{equation}\label{eq:hp-complexity}
\tilde O\left(\frac{\alpha_{\mathrm{s}}^2(1 + \log(1/\eta))^3(\kappa_{\mathrm{s}} + \log(1/\eta))\bigl(\sigma_{\mathrm{s}}^2K + L_{\mathrm{s}}^2K^3\bigr)}{\epsilon^3}\right).
\end{equation}
\end{theorem}

The theorem shows that requiring confidence $1-\eta$ does not change the leading $\epsilon^{-3}$ dependence. Confidence and smoothing enter only through logarithmic factors and the smoothness and norm-comparison constants of the auxiliary norm. The proof follows the same structure as the expected Banach result, but controls all estimator errors and increments simultaneously on an event of probability at least $1-\eta$, rather than only in second moment. The smooth-norm Poisson decomposition establishes these bounds for every refresh and difference block. On this event, the deterministic displacement-level inequality converts the increment bounds into the displayed residual estimate. The block tail bound and complete proofs are given in \cref{app:hp-details}.

\section{Experiments}\label{sec:experiments}

We use two eight-state average-reward problems driven by the same continuing cycle chain. In the first, policy evaluation has the linear Bellman map $Bh=\bar r+Ph$, which is non-expansive in the Euclidean norm. In the second, a two-action control problem has the nonlinear optimality map $\mathcal TQ=\bar r+P\max_a Q(\cdot,a)$, which is non-expansive in the sup norm but not in any inner-product norm. Both transition oracles can be evaluated at two query points after observing a single transition. The constructions, fixed points, and unbiasedness calculations are given in \cref{app:experiments}.

We compare block Halpern with the corresponding Hilbert or Banach PAGE-Halpern method on fast- and slow-mixing versions of the chain. Each pair receives the same realized trajectory, beginning at state zero, and every observed transition counts toward the budget. We report the normalized last-iterate residual at common budgets $500\cdot2^j$, $j=0,\ldots,12$, over twenty repetitions. Lines in \cref{fig:average-reward} show medians and the shaded regions show interquartile ranges, which summarize the central half of the trajectories without allowing a few large residuals to dominate. Hyperparameters and target-hitting transition counts are reported in \cref{app:experiments}.

\begin{figure}[t]
\centering
\includegraphics[width=.92\textwidth]{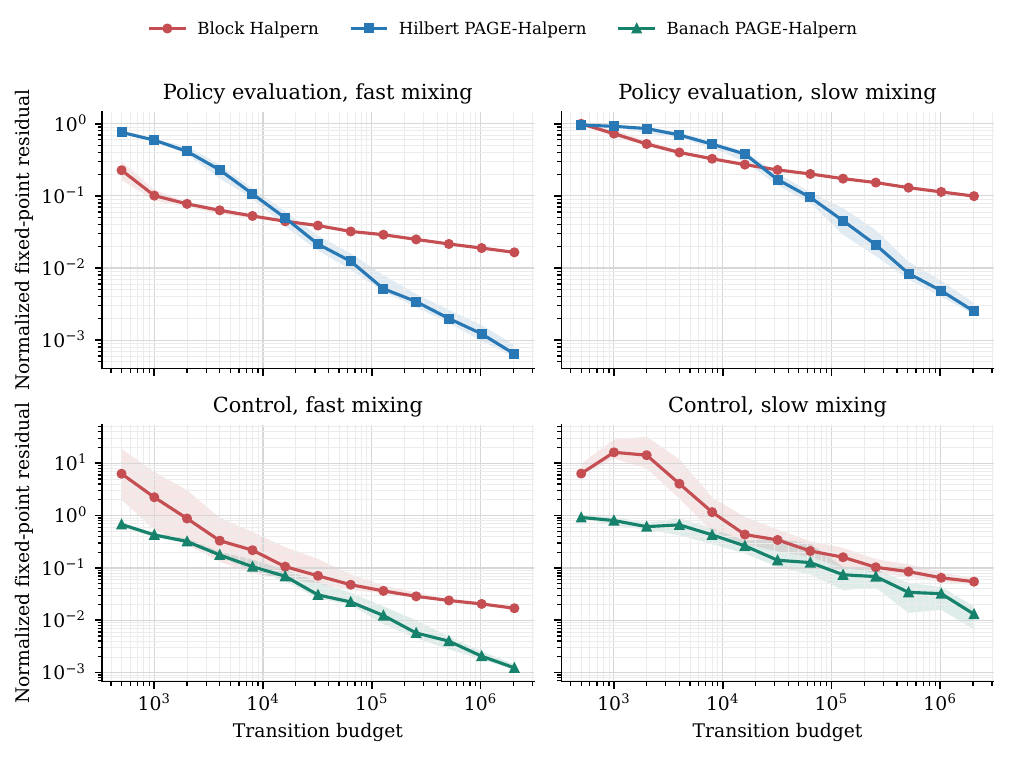}
\caption{Normalized fixed-point residual at common transition budgets. The top row evaluates the Hilbert method on policy evaluation, and the bottom row evaluates the Banach method on average-reward control. Each PAGE method is compared with the block method on the same realized trajectories. Lines show medians over twenty trajectories and shaded regions interquartile ranges.}
\label{fig:average-reward}
\end{figure}

On the fast evaluation chain, block Halpern benefits initially from its short blocks, but PAGE is more accurate once its initialization cost is amortized. At $128{,}000$ transitions, its median residual is $0.00515$, compared with $0.0290$ for block Halpern. The advantage is clearer on the slow chain, where PAGE reaches residual $0.05$ after a median of $107{,}374$ transitions and the block method does not reach that target within the maximum budget. The Banach control experiment gives the same qualitative separation. On the slow chain, Banach PAGE reaches residual $0.05$ after a median of $80{,}000$ transitions in all repetitions, whereas block Halpern reaches it in only 8/20. Thus, same-state differences become increasingly valuable as the target tightens or the trajectory mixes more slowly.

\section{Conclusion}

We developed a finite-sample theory for stochastic Halpern iteration along one continuing Markovian trajectory. A direct block method gives an expected $O(\log N/N)$ last-iterate residual with $\tilde O(\epsilon^{-5})$ samples, while same-state recursive variance estimation improves the complexity to $O(\epsilon^{-3})$ in Hilbert space and $\tilde O(\epsilon^{-3})$ in finite-dimensional Banach spaces. The Banach analysis works in the norm where $T$ is non-expansive, and smooth auxiliary norms give high-probability guarantees for sup and block-sup geometries. The theory makes the required information structure explicit. Stability controls the visited region, the Poisson equation handles Markovian dependence and bias, and PAGE requires each sampled state to be reused at two query points. The average-reward experiments implement this model on continuing trajectories without stationary-oracle access. Thus, Halpern anchoring and same-state variance reduction provide last-iterate guarantees in the Banach geometries natural to RL.

\newpage

\subsection*{AI use statement}
Generative AI tools assisted with literature discovery, designing and implementing the synthetic experiments, and language and formatting revisions. The authors independently checked all citations. The authors made all final scientific and editorial decisions and take responsibility for the complete contents of the submission.

\subsection*{Ethics statement}
This work is theoretical and uses only synthetic finite-state reinforcement learning problems. It involves no human subjects, personal or private data, external datasets, or deployed decision systems, and raises no specific ethical concerns.

\subsection*{Reproducibility statement}
The main text states all assumptions, oracle conditions, algorithms, and block-size schedules. Complete proofs are provided in the appendix. In particular, \cref{app:baseline-details} gives the stability and baseline analysis, \cref{app:hilbert-details} develops the stopping-time Poisson decomposition and Hilbert-space block estimates, and \cref{app:banach-details,app:hp-details} contain the Banach-space and high-probability arguments. The experiments use no external data. \Cref{app:experiments} specifies both finite-state problems and their transition oracles, all initialization and hyperparameter choices, the sampling protocol, transition budgets, residual metrics, number of repetitions, and reporting procedure. The supplementary material includes the code used to run the experiments and generate the reported figure and transition-count data.

\newpage

\bibliography{iclr2027_conference}
\bibliographystyle{iclr2027_conference}

\newpage

\appendix
\crefalias{section}{appendix}
\crefalias{subsection}{appendix}
\crefalias{subsubsection}{appendix}
\crefname{appendix}{Appendix}{Appendices}
\Crefname{appendix}{Appendix}{Appendices}
\section*{Appendix Table of Contents}
\addcontentsline{toc}{section}{Appendix}
\markboth{Appendix}{Appendix}
\startcontents[appendix]
\printcontents[appendix]{l}{1}{\setcounter{tocdepth}{3}}
\newpage

\section{Additional related work}\label{app:related}

The classical averaged iterations of \citet{mann1953mean} and \citet{krasnoselskii1955two} form the basis of Krasnosel'ski\u{i}--Mann methods for non-expansive fixed-point problems. A modern account of their relation to monotone operators and splitting methods is given by \citet{bauschke2020correction}. Halpern iteration introduces an anchor into this recursion \citep{halpern1967fixed}. The strong convergence of this iteration to the fixed point selected by the anchor was established under standard weight conditions by \citet{wittmann1992approximation}. For the canonical weights in Hilbert space, \citet{lieder2020convergence} proved a tight $O(1/N)$ residual bound. \citet{diakonikolas2020halpern} used this rate and a residual potential to obtain near-optimal methods for monotone inclusions and variational inequalities. We use the same anchored last-iterate structure, but develop perturbation estimates that remain valid for stochastic operator evaluations.

For stochastic non-expansive maps, \citet{bravo2024stochastic} analyze Krasnosel'ski\u{i}--Mann iterations under martingale noise. \citet{bravo2026stochastic} study stochastic Halpern iteration in normed spaces under i.i.d.\ oracle access. Their explicit minibatch construction has $O(\epsilon^{-5})$ complexity, while their lower bound has order $\Omega(\epsilon^{-3})$ for a broad single-point linear-span oracle model. However, the same-sample multi-point access used by PAGE is stronger than their lower-bound oracle. The lower bound is therefore an important benchmark, but does not constitute a matched minimax statement for our information model. Recent stochastic Halpern analyses also study independent minibatches and abstract asymptotic regularity \citep{iiduka2026mini,pischke2026asymptotic}. \citet{cai2022stochastic} combine stochastic Halpern iteration with recursive variance reduction for monotone inclusions under independent sampling.

The analysis of stochastic approximation with Markovian observations has a long history. The ODE and Poisson-equation approaches are developed systematically by \citet{kushner2003stochastic},  \citet{borkar2008stochastic}, and \citet{benveniste2012adaptive}. Modern finite-time results cover linear stochastic approximation and temporal-difference learning \citep{srikant2019finite,durmus2021stability}, contractive operators in arbitrary norms \citep{chen2020finite,chen2024lyapunov,qu2020finite}, and unbounded Markovian noise under Lyapunov growth conditions \citep{haque2025stochastic}. In the non-expansive setting, \citet{blaser2026asymptotic} develop asymptotic and finite-sample guarantees for stochastic Krasnosel'ski\u{i}--Mann recursion with Markovian noise. Our blocks also preserve one continuing trajectory. Their boundaries are predictable stopping times, and the Poisson equation controls both the conditional bias and the second moment of each estimator.

Recursive variance reduction was developed for stochastic optimization through methods including SVRG, SARAH, SPIDER, and PAGE \citep{johnson2013accelerating,nguyen2017sarah,fang2018spider,li2021page}. These methods exploit same-sample differences whose second moments decrease as successive query points become close. Related constructions have been developed for monotone variational inequalities \citep{alacaoglu2022stochastic} and for stochastic Halpern iteration under independent sampling \citep{cai2022stochastic}. Variance reduction from a single Markov trajectory has also been studied for smooth nonconvex optimization and temporal-difference learning \citep{xureanalysis, sun2026learning}. Our estimator uses the same recursive principle, while its refresh and difference samples are consecutive segments of one Markov chain.

The Banach-space and high-probability parts require a separate geometric argument. Smooth Lyapunov functions based on generalized Moreau envelopes give finite-sample stochastic approximation bounds in arbitrary contraction norms \citep{chen2020finite,chen2024lyapunov}. Concentration for vector-valued martingales in smooth Banach spaces is developed by \citet{pinelis1994optimum} and \citet{juditsky2023largedeviationsvectorvaluedmartingales}. Our expected Banach result does not smooth the fixed-point residual, but instead replaces the Hilbert cocoercive potential with a displacement estimate in the original non-expansiveness norm. Smoothing enters only when the recursive estimator must be controlled with high probability, where we use the unified estimator analysis of \citet{luo2026unified}.

Average-reward policy evaluation is a natural continuing-trajectory application. Classical temporal-difference methods estimate differential values from one endless Markov trajectory \citep{tsitsiklis1999average}, while average-reward control leads to relative-value and $Q$-learning recursions whose Bellman maps are naturally studied modulo additive constants \citep{abounadi2001learning,wan2021learning}. Concurrently, \citet{lee2026learning} combine anchored value iteration with recursive Bellman-difference estimation for finite weakly communicating average-reward MDPs and obtain a high-probability $\tilde O(\epsilon^{-2})$ transition complexity for finding an $\epsilon$-optimal policy. Their construction uses return times to extract state-conditional i.i.d.\ transition samples from one trajectory and exploits the exact Bellman decomposition. Our analysis instead controls consecutive Markovian blocks through the Poisson equation and applies to general non-expansive operators in finite-dimensional Banach spaces. Distributional RL replaces scalar values by return distributions and uses supremum-over-state metrics for its Bellman operators \citep{bellemare2017distributional,rowland2018analysis,bellemare2023distributional}. Recent finite-iteration results study asynchronous categorical TD under Markovian sampling in statewise supremum geometries \citep{kaya2026finite}, while quotient-categorical representations yield non-expansive average-reward distributional operators in a coordinate Cram\'er geometry \citep{kaya2026quotient}. These settings motivate a last-iterate theory that does not replace the application norm by an inner-product norm.

\section{Perturbation-level analysis and the Markovian block baseline}\label{app:baseline-details}

\subsection{Discussion of assumptions}
\Cref{ass:oracle-variance} is the Markovian analogue of the pointwise finite-variance oracle model used for stochastic Halpern iteration in \citet[Assumption 1]{bravo2026stochastic}. Related finite-sample stochastic approximation analyses impose the same kind of second-moment control, often in a growth form. See, for instance, \citet[Assumption 2.2]{chen2020finite} and \citet[Assumption 4]{chen2024lyapunov}. The explicit complexity theorem below states the additional growth bound needed to turn this oracle model into exact numerical constants.

\Cref{ass:oracle-stable} is the standard stochastic-approximation stability condition. See, for instance, the stability theory of \citet{borkar2008stochastic} and \citet{borkar2000ode}. In the Markovian non-expansive setting, \citet{blaser2026asymptotic} identify stability through the boundedness of the iterates. The non-expansive fixed-point residual analyses of \citet[Condition (H4)]{bravo2024stochastic} use the same type of stability condition for stochastic Krasnosel'ski\u{i}--Mann iterations, while \citet[Lemma 1]{bravo2026stochastic} use the closely related boundedness of the operator images $T x_n - x_0$. \Cref{lem:stability-constants} below shows that this operator-image bound is a consequence of \cref{ass:oracle-stable,ass:nonexpansive}.

\begin{remark}[A sufficient condition for stability]
Suppose that, for $q \in \{1, 2\}$, there exist constants $K_q < \infty$ such that
\begin{equation}
\sup_{x, y} \norm{H(x, y) - x_0}^q \le K^q_q.
\end{equation}
Then \cref{ass:oracle-stable} holds for any update
\begin{equation}
x_n = (1-\beta_n)x_0+\beta_nG_n
\end{equation}
with $\beta_n\in[0,1]$ in which $G_n$ is a convex combination of oracle values $H(x_{n-1},y)$. Indeed, if $G_n=\sum_j w_jH(x_{n-1},y_j)$ with $w_j\ge0$ and $\sum_jw_j=1$, then Jensen's inequality gives
\begin{equation}
\norm{x_n-x_0}^q = \beta_n^q\norm{G_n-x_0}^q \le \sum_jw_j\norm{H(x_{n-1},y_j)-x_0}^q \le K_q^q.
\end{equation}
Thus, \cref{ass:oracle-stable} holds with $R_q = K_q$. In particular, this condition applies to the Markovian block recursion in \cref{sec:baseline}.
\end{remark}

\subsection{Perturbation-level Halpern bounds}\label{app:inexact-details}
For $1 \le i \le n$, let
\begin{equation}
B^n_i := \prod_{j= i}^n \beta_j,
\end{equation}
with the empty-product convention $B^n_i := 1$ when $0 \le n <i$, and define $\beta_0 := \sigma_0 := 0$, where $\sigma_n := \E \bigl[\norm{U_n} \bigr]$.

\begin{lemma}[Stability constants]\label{lem:stability-constants}
If \cref{ass:oracle-stable} holds, then
\begin{equation}
\sup_{n \ge 0} \E\bigl[\norm{T x_n - x_0}\bigr] \le \kappa_1, \qquad \sup_{n \ge 0} \E \bigl[\norm{T x_n - x_0}^2\bigr] \le \kappa^2_2.
\end{equation}
More precisely, if $D_0 := \dist(x_0, \Fix(T))$, then we may take
\begin{equation}
\kappa_q := R_q + 2 D_0, \qquad q \in \{1, 2\}.
\end{equation}
\end{lemma}

\begin{proof}[Proof of \cref{lem:stability-constants}]
Choose $x^* \in \Fix(T)$ such that $\norm{x^* - x_0} = D_0$. The choice exists because $\Fix(T)$ is closed in finite dimension. By non-expansiveness and $Tx^* = x^*$,
\begin{equation}
\norm{Tx_n - x_0} \le \norm{T x_n - T x^*} + \norm{x^* - x_0} \le \norm{x_n - x^*} + D_0 \le \norm{x_n - x_0} + 2 D_0.
\end{equation}
Taking expectations gives the $q=1$ claim with $\kappa_1 = R_1 + 2D_0$. Taking second moments in the same inequality and then applying Minkowski's inequality gives the $q=2$ claim with $\kappa_2 = R_2 + 2 D_0$.
\end{proof}

\begin{proposition}[Perturbation-level inexact Halpern residual]\label{prop:pert-inexact-halpern}
Suppose \cref{ass:nonexpansive,ass:oracle-stable} hold, and let $(\beta_n)$ be nondecreasing. Let $(x_n)$ be generated by \eqref{eq:recursion-with-error}. Then, for all $n \ge 1$,
\begin{equation}
\E \bigl[\norm{x_n - T x_n} \bigr] \le \kappa_1 (1 - \beta_n) + \sum_{i =1}^n B^n_i \bigl( \kappa_1 (\beta_i - \beta_{i-1}) + \beta_i\sigma_i + \beta_{i-1}\sigma_{i-1}\bigr) + \beta_n \sigma_n.
\end{equation}
\end{proposition}

\begin{proof}[Proof of \cref{prop:pert-inexact-halpern}]
Let $r_n := \norm{T x_n - x_0}$, $P_0 := 0$, and, for $n \ge 1$,
\begin{equation}
P_n := r_{n-1}(\beta_n - \beta_{n-1}) + \beta_{n-1}P_{n-1} + \beta_n \norm{U_n} + \beta_{n-1} \norm{U_{n-1}}.
\end{equation}
We have
\begin{equation}
x_n - x_{n-1} = (\beta_n - \beta_{n-1})(Tx_{n-1} - x_0) + \beta_{n-1}(Tx_{n-1}-Tx_{n-2}) + \beta_n U_n - \beta_{n-1}U_{n-1}.
\end{equation}
Hence, using the non-expansiveness in \cref{ass:nonexpansive},
\begin{equation}\label{eq:pn}
\begin{aligned}
\norm{x_n - x_{n-1}} &\le (\beta_n - \beta_{n-1}) r_{n-1}+ \beta_{n-1} \norm{x_{n-1} - x_{n-2}} + \beta_n \norm{U_n} + \beta_{n-1}\norm{U_{n-1}}. \\
&= P_n.
\end{aligned}
\end{equation}
Unrolling the recursion for $P_n$, we obtain
\begin{equation}
P_n = \sum_{i=1}^n B^{n-1}_i \bigl(r_{i-1}(\beta_i - \beta_{i-1}) + \beta_i \norm{U_i} + \beta_{i -1}\norm{U_{i-1}}\bigr).
\end{equation}
Now, we have
\begin{equation}
\begin{aligned}
\norm{x_n - T x_n} &\le (1 - \beta_n) \norm{x_0 - T x_n} + \beta_n \norm{T x_{n-1} - T x_n} + \beta_n \norm{U_n} \\
&\le (1 - \beta_n) r_n + \beta_n \norm{T x_{n-1} - Tx_n} + \beta_n \norm{U_n} \\
&\le (1 - \beta_n) r_n + \beta_n \norm{x_{n-1} - x_n} + \beta_n \norm{U_n} \\
&\le (1 - \beta_n) r_n + \beta_n P_n + \beta_n \norm{U_n},
\end{aligned}
\end{equation}
using the definition of $r_n$, non-expansiveness of $T$, and \eqref{eq:pn}. Plugging in the definition of $P_n$ and using the fact that $\beta_n B^{n-1}_i = B^n_i$, and taking expectations,
\begin{equation}
\E \bigl[\norm{x_n - Tx_n} \bigr] \le (1 - \beta_n) \E[r_n] + \sum_{i =1}^n B^n_i \bigl(\E [r_{i-1}](\beta_i - \beta_{i-1}) + \beta_i \E\bigl[ \norm{U_i}\bigr] + \beta_{i-1} \E\bigl[\norm{U_{i-1}}\bigr]\bigr) + \beta_n \E \bigl[ \norm{U_n}\bigr].
\end{equation}
By \cref{ass:oracle-stable}, $\sup_{n \ge 0} \E[r_n] = \sup_{n \ge 0} \E\bigl[ \norm{T x_n - x_0}\bigr] \le \kappa_1$. Then,
\begin{equation}
\E \bigl[\norm{x_n - Tx_n} \bigr] \le \kappa_1(1 - \beta_n) + \sum_{i =1}^n B^n_i \bigl(\kappa_1(\beta_i - \beta_{i-1}) + \beta_i \sigma_i + \beta_{i-1} \sigma_{i-1}\bigr) + \beta_n \sigma_n.
\end{equation}
\end{proof}

\begin{lemma}[Perturbation-level bound]\label{lem:perturbation-level}
Under the conditions of \cref{prop:pert-inexact-halpern}, let $\beta_n = n/(n+1)$ and $N \in \N$. Then,
\begin{equation}
\E\bigl[\norm{x_N - T x_N} \bigr] \le \frac{\kappa_1(1 + \log(N+1))}{N+1} + \frac{2}{N+1}\sum_{n=1}^N n \sigma_n.
\end{equation}
Consequently, if $\sum_{n=1}^N n \sigma_n \le C_\sigma(1 + \log(N+1))$, then
\begin{equation}
\E\bigl[\norm{x_N - T x_N}\bigr] \le \frac{(\kappa_1 + 2 C_\sigma)(1 + \log(N+1))}{N+1}.
\end{equation}
\end{lemma}

\begin{proof}[Proof of \cref{lem:perturbation-level}]
\Cref{prop:pert-inexact-halpern} directly gives
\begin{equation}
\E \bigl[\norm{x_N - Tx_N} \bigr] \le \kappa_1(1 - \beta_N) + \sum_{i =1}^N B^N_i \bigl(\kappa_1(\beta_i - \beta_{i-1}) + \beta_i \sigma_i + \beta_{i-1} \sigma_{i-1}\bigr) + \beta_N \sigma_N.
\end{equation}
With the choice of the step size sequence, we get
\begin{equation}
\kappa_1(1 - \beta_N) = \kappa_1\left(1 - \frac{N}{N+1}\right) = \frac{\kappa_1}{N+1}.
\end{equation}
Also,
\begin{equation}
B^N_i = \prod_{j=i}^N \frac{j}{j+1} = \frac{i}{N+1}
\end{equation}
by means of a telescoping product. Finally,
\begin{equation}
\beta_i - \beta_{i-1} = \frac{i}{i+1} - \frac{i-1}{i} = \frac{1}{i(i+1)}.
\end{equation}
Therefore, the exact-Halpern part of the bound becomes
\begin{equation}
\frac{\kappa_1}{N+1}+\sum_{i=1}^N B^N_i \kappa_1(\beta_i - \beta_{i-1}) = \frac{\kappa_1}{N+1}+\frac{\kappa_1}{N+1}\sum_{i=1}^N \frac{1}{i+1} = \frac{\kappa_1}{N+1}\sum_{i=0}^N\frac{1}{i+1}.
\end{equation}
Since $\sum_{i=0}^N \frac{1}{i+1} \le 1 + \log(N+1)$, the exact-Halpern terms are bounded by
\begin{equation}
\frac{\kappa_1(1+\log(N+1))}{N+1}.
\end{equation}
We now need to handle the perturbation terms, i.e.,
\begin{equation}
\sum_{i=1}^N B^N_i (\beta_i \sigma_i + \beta_{i-1} \sigma_{i-1}) + \beta_N \sigma_N.
\end{equation}
We separate the two sums and shift the index of the second:
\begin{equation}
\begin{aligned}
\sum_{i=1}^N B^N_i (\beta_i \sigma_i + \beta_{i-1} \sigma_{i-1}) + \beta_N \sigma_N &= \sum_{i=1}^N B^N_i \beta_i \sigma_i + \sum_{i=1}^N B^N_i\beta_{i-1} \sigma_{i-1} + \beta_N \sigma_N \\
& = \sum_{i=1}^{N-1} \beta_i (B^N_i + B^N_{i+1}) \sigma_i + \beta_N (B^N_N + 1) \sigma_N.
\end{aligned}
\end{equation}
Using $B^N_i = i /(N+1)$, each $1 \le i < N$ satisfies
\begin{equation}
\beta_i(B^N_i + B^N_{i+1}) = \frac{i}{i+1} \frac{2i+1}{N+1} \le \frac{2i}{N+1},
\end{equation}
and
\begin{equation}
\beta_N(B^N_N + 1) = \frac{N}{N+1} \left( \frac{N}{N+1} + 1 \right) \le \frac{2N}{N+1}.
\end{equation}
Substituting these into \cref{prop:pert-inexact-halpern},
\begin{equation}
\E\bigl[\norm{x_N - T x_N} \bigr] \le \frac{\kappa_1(1 + \log(N+1))}{N+1} + \frac{2}{N+1}\sum_{n=1}^N n \sigma_n.
\end{equation}
\end{proof}

\subsection{Markovian block method}\label{sec:baseline}

Fix an anchor $x_0 \in \R^d$, a step size sequence $(\beta_n) \subset (0, 1)$, block lengths $(k_n) \subset \N$, and optional burn-in lengths $(b_n) \subset \N \cup \{0\}$. At iteration $n$, starting from the current Markov state $Y_{\tau_n}$, run the chain for $b_n$ unused steps, and then collect the next $k_n$ states. Let
\begin{equation}
\tau_1 := 0, \qquad \tau_n := \sum_{i=1}^{n-1} (b_i + k_i) \quad \text{for } n \ge 2,
\end{equation}
so that $\tau_n$ indicates the time at which the $n$th block begins. Further, let
\begin{equation}
G_n := \frac{1}{k_n} \sum_{j=1}^{k_n} H(x_{n-1}, Y_{\tau_n + b_n + j}).
\end{equation}
The Markovian-block Halpern update is
\begin{equation}\label{eq:recursion}
x_n = (1- \beta_n) x_0 + \beta_n G_n, \qquad n \ge 1.
\end{equation}
Equivalently,
\begin{equation}
x_n = (1 - \beta_n) x_0 + \beta_n (T x_{n-1} + U_n), \qquad U_n:= G_n - T x_{n-1}.
\end{equation}
We can interpret $U_n$ as the block estimation error, which includes both the error due to sampling and the Markovian bias from using a dependent, nonstationary segment of the chain.

A mixing-time argument would first discard a burn-in segment to make the conditional bias small. Instead, we use the Poisson-equation decomposition of Markovian noise \citep{haque2025stochastic,blaser2026asymptotic}. The conditional bias is then a boundary term of order $k^{-1}$, while the conditional second moment is of order $k^{-1}$. Thus, a block can begin immediately at the current stopping time.

\begin{lemma}[Markovian block error]\label{lem:markov-block-error}
Suppose \cref{ass:ergodic} holds. There exists a constant $C_{\mathrm{blk}}$, depending only on the finite-state chain and $\mu_{\norm{\cdot}}$, such that the following holds. Let $\tau$ be a stopping time and let $x$ be an $\cF_\tau$-measurable query point. Define
\begin{equation}
V_H(x) := \sum_{y \in \cY} \pi(y) \norm{H(x,y) - Tx}^2_2.
\end{equation}
Then, for every integer-valued $\cF_\tau$-measurable block length $k \ge 1$,
\begin{equation}
\E\Bigl[\Bigl\lVert\frac{1}{k}\sum_{j=1}^k H(x, Y_{\tau+j}) - Tx\Bigr\rVert \mid \cF_\tau\Bigr] \le \frac{C_{\mathrm{blk}} V_H(x)^{1/2}}{\sqrt{k}}.
\end{equation}
\end{lemma}

\begin{proof}[Proof of \cref{lem:markov-block-error}]
Condition on $\cF_\tau$. The query point $x$ and the block length $k$ are then fixed. Let
\begin{equation}
f_x(y) := H(x, y) - Tx.
\end{equation}
By definition, $Tx = \sum_{y \in \cY} \pi(y) H(x, y)$, so $f_x$ has stationary mean zero and
\begin{equation}
\norm{f_x}_{2,\pi}^2 = \sum_{y \in \cY} \pi(y)\norm{f_x(y)}_2^2 = V_H(x).
\end{equation}
By the Poisson block estimate in \cref{lem:poisson-block},
\begin{equation}
\E\Bigl[\Bigl\lVert\frac{1}{k}\sum_{j=1}^k f_x(Y_{\tau+j})\Bigr\rVert_2^2 \mid \cF_\tau\Bigr] \le \frac{C_{\mathrm V}V_H(x)}{k}.
\end{equation}
Conditional Jensen's inequality and $\norm{z} \le \mu_{\norm{\cdot}}\norm{z}_2$ give the result after absorbing $\mu_{\norm{\cdot}}\sqrt{C_{\mathrm V}}$ into $C_{\mathrm{blk}}$.
\end{proof}

Although \cref{ass:oracle-variance} gives us a qualitative property of the oracle model, in that it requires finite variance at each fixed query point, establishing a finite-sample rate further requires a quantitative way to control the variances at the random query points generated by the algorithm. The affine growth condition that follows, together with trajectory stability, provides such a bound.

\begin{condition}[Affine variance growth]\label{cond:affine-var}
Every query point $x$ generated up to the target horizon satisfies
\begin{equation}
\sum_{y \in \cY} \pi(y) \norm{H(x, y) - Tx}^2_2 \le \sigma^2_0 + \sigma^2_1 \norm{Tx - x_0}^2.
\end{equation}
\end{condition}

Uniformly bounded variance is a special case that is recovered when $\sigma_1 = 0$. This distinction is the same one present in the i.i.d.\ stochastic Halpern analysis of \citet{bravo2026stochastic}. Their standing oracle assumption is pointwise finite variance, but their concrete non-expansive minibatch complexity results require uniformly bounded variance in order to turn the abstract perturbation condition into an explicit schedule for the batch sizes. \Cref{cond:affine-var} relaxes that uniformly bounded variance specialization while still yielding an explicit finite-sample constant through \cref{ass:oracle-stable}.

The affine-growth condition is also natural in applications such as RL. In TD or distributional TD updates, the random target contains the current value or return-distribution estimate at a sampled next state. Consequently, the conditional second moment of the update noise typically scales with the size of the current estimate. This is the reason why standard stochastic approximation analyses of RL use affine-growth variance conditions. See, for instance, the discussion by \citet[Section 6.8, pp. 180-181]{bellemare2023distributional} and the stochastic approximation stability theory of \citet{borkar2008stochastic} and \citet{borkar2000ode}.

The next theorem quantifies the best guarantee obtained by rebuilding the operator estimate from a fresh Markovian block at every Halpern iteration. It serves as the non-variance-reduced reference point for the PAGE results in the main body.

\begin{theorem}[Inexact Halpern residual with Markovian block minibatches]\label{thm:block-baseline} Suppose \cref{ass:nonexpansive,ass:ergodic,ass:oracle-variance,ass:oracle-stable,cond:affine-var} hold. Run \eqref{eq:recursion} with
\begin{equation}
\beta_n = \frac{n}{n+1}, \qquad k_n = n^4,
\end{equation}
and let $(b_n)$ be any nonnegative burn-in schedule satisfying $\sum_{n=1}^N b_n = O(N\log N)$. In particular, the no burn-in scenario with $b_n=0$ is also admissible. Let $D_0 := \dist(x_0, \Fix(T))$. Then, there exists a constant $C$, depending only on $R_1, R_2, D_0, C_{\mathrm{blk}}, \sigma_0, \sigma_1$ and $\rho$ such that
\begin{equation}
\E \bigl[ \norm{x_N - T x_N}\bigr] \le C \frac{1 + \log(N+1)}{N+1}.
\end{equation}
Consequently, to obtain $\E\bigl[\norm{x_N - T x_N}\bigr] \le \epsilon$, it suffices to take
\begin{equation}
N = O\left(\frac{C}{\epsilon}\log\frac{C}{\epsilon}\right),
\end{equation}
and the total number of Markov samples needed is
\begin{equation}
\sum_{n=1}^N (b_n + k_n) = \tilde O(\epsilon^{-5}).
\end{equation}
\end{theorem}

The theorem gives the desired $O(\log N/N)$ last-iterate residual, but its sample cost is substantially larger. The perturbation bound requires the $n$th block error to be of order $n^{-2}$, which in turn requires $k_n=n^4$ samples. Summing these costs over $N$ iterations produces the fifth-order dependence on the target accuracy. This is the bottleneck removed by recursive variance estimation.

\begin{proof}[Proof of \cref{thm:block-baseline}]
Let $\cF_{\tau_n}$ denote the sigma-algebra at the start of the $n$th block. By construction, $x_{n-1}$, $b_n$ and $k_n$ are $\cF_{\tau_n}$-measurable, while the states used in $G_n$ are sampled after step $\tau_n$. Therefore, \cref{lem:markov-block-error} applies after conditioning on $\cF_{\tau_n}$, with query point $x = x_{n-1}$, burn-in length $b = b_n$, and retained block length $k = k_n$. By \cref{cond:affine-var} and \cref{ass:oracle-stable},
\begin{equation}
\E\bigl[ V_H(x_{n-1})\bigr] \le \sigma^2_0 + \sigma^2_1 \E \bigl[ \norm{Tx_{n-1} - x_0}^2\bigr] \le \bar \sigma^2_H,
\end{equation}
where $\bar \sigma^2_H := \sigma^2_0 + \sigma^2_1 \kappa^2_2$. By \cref{lem:markov-block-error},
\begin{equation}
\E\bigl[\norm{U_n} \mid \cF_{\tau_n}\bigr] \le C_{\mathrm{blk}} V_H(x_{n-1})^{1/2}k_n^{-1/2}.
\end{equation}
Taking expectations using Jensen's inequality, and the choice $k_n = n^4$, we get
\begin{equation}
\sigma_n = \E\bigl[\norm{U_n}\bigr] \le C_{\mathrm{blk}}\bar\sigma_Hn^{-2} \le C_0 n^{-2}
\end{equation}
for some constant $C_0$. Therefore, the weighted perturbation sum required by \cref{lem:perturbation-level} satisfies
\begin{equation}
\sum_{n=1}^N n\sigma_n \le C_0 \sum_{n=1}^N \frac1n \le C_0(1 + \log N).
\end{equation}
For the sample complexity, each iteration uses exactly $b_n + k_n$ Markovian samples. By the assumed bound on the burn-in schedule,
\begin{equation}
\sum_{n=1}^N (b_n + k_n) = O(N^5) + O(N\log N) = O(N^5).
\end{equation}
Substituting $N = O\bigl((C/\epsilon)\log(C/\epsilon)\bigr)$ gives a sample complexity of $\tilde O(\epsilon^{-5})$.
\end{proof}

\section{Supporting results for Hilbert PAGE-Halpern}\label{app:hilbert-details}
\subsection{Poisson decomposition of Markovian error}

The i.i.d.\ analyses of stochastic approximation methods leverage the martingale-difference property of the sequence of perturbations. This property is evidently not available when the samples are Markovian. However, an analogue is recovered through the Poisson-equation decomposition \citep{blaser2026asymptotic}. For any stationary-mean zero function $g: \cY \to \R^d$, finiteness and ergodicity of the chain guarantee a solution $\nu_g$ to
\begin{equation}
g(y) = \nu_g(y) - (P \nu_g)(y).
\end{equation}
Along a block beginning at $\tau$, this identity can be written as
\begin{equation}
g(Y_{\tau+i}) = M_i + (P\nu_g)(Y_{\tau+i-1}) - (P\nu_g)(Y_{\tau+i}),
\end{equation}
where
\begin{equation}
M_i = \nu_g(Y_{\tau + i}) - (P\nu_g)(Y_{\tau+i-1})
\end{equation}
is a martingale difference after conditioning on $\cF_\tau$. Summing over a block leaves a martingale sum and two boundary terms. This gives the usual $S^{-1}$ second-moment scaling for the random part of the block average, together with an additional $S^{-1}$ conditional bias from the boundary terms. \Cref{lem:markov-block-error} used the resulting first-moment estimate for the baseline method. We now state the full conditional bounds and specialize them to the PAGE refresh and difference estimators.

\begin{lemma}\label{lem:poisson-block}
Suppose \cref{ass:ergodic} holds. Let $g:\cY\to\R^d$ satisfy $\sum_y\pi(y)g(y)=0$. There exist constants $C_{\mathrm B},C_{\mathrm V}$, depending only on the finite-state Markov chain, such that for any stopping time $\tau$ and any integer-valued $\cF_\tau$-measurable block length $S \ge 1$,
\begin{equation}
\Bigl\lVert\E\Bigl[\frac{1}{S}\sum_{i=1}^S g(Y_{\tau+i})\mid\cF_\tau\Bigr]\Bigr\rVert_2 \le \frac{C_{\mathrm B}\norm{g}_{2,\pi}}{S},
\end{equation}
and
\begin{equation}
\E\Bigl[\Bigl\lVert\frac{1}{S}\sum_{i=1}^S g(Y_{\tau+i})\Bigr\rVert_2^2\mid\cF_\tau\Bigr] \le \frac{C_{\mathrm V}\norm{g}_{2,\pi}^2}{S}.
\end{equation}
\end{lemma}

\begin{proof}[Proof of \cref{lem:poisson-block}]
Condition on $\cF_\tau$.  Since $S$ is $\cF_\tau$-measurable, it is fixed before the block is sampled.  It is therefore enough to prove the displayed inequalities for deterministic $S$, with constants independent of $\tau$ and $S$. Since $g$ has stationary mean zero and the chain is finite and geometrically mixing, the Poisson equation
\begin{equation}
g(y)=\nu_g(y)-(P\nu_g)(y)
\end{equation}
has the solution
\begin{equation}
\nu_g(y)=\sum_{t=0}^\infty (P^tg)(y).
\end{equation}
Finiteness and ergodicity give a uniform Poisson-solution bound: using the displayed geometric bound from \cref{ass:ergodic} and $\norm{g}_{\infty,2}\le \pi_{\min}^{-1/2}\norm{g}_{2,\pi}$, we have, for some constant $C_\nu$
\begin{equation}
\norm{\nu_g}_{\infty,2} \le C_\nu\norm{g}_{2,\pi}.
\end{equation}
The constant $C_\nu$ is uniform over all centered functions $g$, it depends only on the transition matrix and the stationary distribution.  Define
\begin{equation}
M_i = \nu_g(Y_{\tau+i})-(P\nu_g)(Y_{\tau+i-1}).
\end{equation}
By the strong Markov property at $\tau$, $(M_i)_{i=1}^S$ is a martingale difference sequence when conditioned on $\cF_\tau$, since
\begin{equation}
\E[\nu_g(Y_{\tau+i})\mid\cF_{\tau+i-1}] = (P\nu_g)(Y_{\tau+i-1}).
\end{equation}
Moreover,
\begin{equation}
g(Y_{\tau+i}) = M_i+(P\nu_g)(Y_{\tau+i-1})-(P\nu_g)(Y_{\tau+i}).
\end{equation}
Summing over $i$ telescopes the two boundary terms and yields
\begin{equation}
\sum_{i=1}^S g(Y_{\tau+i}) = \sum_{i=1}^S M_i + (P\nu_g)(Y_\tau) - (P\nu_g)(Y_{\tau+S}).
\end{equation}
Taking conditional expectation removes the martingale sum. The remaining boundary terms combined have norm at most $2C_\nu\norm{g}_{2,\pi}$. After division by $S$ this gives the bias bound with $C_{\mathrm B} = 2C_\nu$.

For the second moment, using the fact that the martingale difference terms are conditionally orthogonal in Euclidean norm,
\begin{equation}
\E\Bigl[\Bigl\lVert\sum_{i=1}^S M_i\Bigr\rVert_2^2 \mid \cF_\tau\Bigr] = \sum_{i=1}^S \E[\norm{M_i}_2^2\mid\cF_\tau] \le 4S C_\nu^2\norm{g}_{2,\pi}^2.
\end{equation}
The boundary terms contribute at most $4C_\nu^2\norm{g}_{2,\pi}^2$. Combining the martingale and boundary contributions with
$\norm{a + b}_2^2 \le 2\norm{a}_2^2 + 2\norm{b}_2^2$ and dividing by $S^2$, we have
\begin{equation}
\E\Bigl[\Bigl\lVert\frac1S \sum_{i=1}^S g(Y_{\tau+i})\Bigr\rVert_2^2 \mid \cF_\tau \Bigr] \le \frac{C_{\mathrm V}\norm{g}_{2,\pi}^2}{S}
\end{equation}
with $C_{\mathrm V} = 16 C^2_\nu$.
\end{proof}

\begin{lemma}\label{lem:page-blocks}
Suppose \cref{ass:ergodic} and \cref{cond:page} hold, and consider the PAGE-Halpern recursion \eqref{eq:page-halpern-1}--\eqref{eq:page-halpern-2}. Let $\tau$ be the current chain time, $x,z \in \R^d$ be $\cF_\tau$-measurable query points, and let $S \ge 1$ be an integer-valued $\cF_\tau$-measurable block length. Then, conditioned on $\cF_\tau$,
\begin{equation}
\begin{gathered}
\bigl\lVert\E\bigl[R(x;\tau,S)-F(x)\mid\cF_\tau\bigr]\bigr\rVert_2 \le \frac{C_{\mathrm B}\sigma_F}{S}, \\
\E\bigl[\bigl\lVert R(x;\tau,S)-F(x)\bigr\rVert_2^2\mid\cF_\tau\bigr] \le \frac{C_{\rm V}\sigma_F^2}{S}, \\
\bigl\lVert \E\bigl[D(x,z;\tau,S)-(F(x)-F(z))\mid\cF_\tau\bigr]\bigr\rVert_2 \le \frac{C_{\mathrm B}L_F\norm{x-z}_2}{S}, \\
\E\bigl[\bigl\lVert D(x,z;\tau,S)-(F(x)-F(z))\bigr\rVert_2^2 \mid \cF_\tau \bigr] \le \frac{C_{\mathrm V}L_F^2\norm{x-z}_2^2}{S}.
\end{gathered}
\end{equation}
\end{lemma}

\begin{proof}[Proof of \cref{lem:page-blocks}]
All query points and the block length are $\cF_\tau$-measurable. Hence, after conditioning on $\cF_\tau$, the functions below and the block length are fixed and \cref{lem:poisson-block} applies directly. For the refresh estimator, apply \cref{lem:poisson-block} to
\begin{equation}
g_x(y) := \widehat F(x,y) - F(x),
\end{equation}
whose stationary mean is zero and whose $L^2(\pi)$ norm is at most $\sigma_F$ by \cref{cond:page}. For the difference estimator, apply \cref{lem:poisson-block} to
\begin{equation}
g_{x,z}(y) := \widehat F(x,y) - \widehat F(z,y) - (F(x) - F(z)).
\end{equation}
This function also has stationary mean zero, and by \cref{cond:page} its $L^2(\pi)$ norm is at most $L_F\norm{x-z}_2$. Substituting these two bounds into \cref{lem:poisson-block} proves the statement. 
\end{proof}

\section{Supporting results for Banach PAGE-Halpern}\label{app:banach-details}
\begin{lemma}\label{lem:page-blocks-banach}
Suppose \cref{ass:ergodic} and \cref{cond:page-banach} hold, and consider the PAGE-Halpern recursion \eqref{eq:page-halpern-banach-1}--\eqref{eq:page-halpern-banach-2}. Let $\tau$ be the current chain time, $x, z \in \R^d$ be $\cF_\tau$-measurable query points, and let $S \ge 1$ be an integer-valued $\cF_\tau$-measurable block length. Then, conditioned on $\cF_\tau$,
\begin{equation}
\begin{gathered}
\bigl\lVert\E\bigl[R_T(x;\tau,S)-Tx\mid\cF_\tau\bigr]\bigr\rVert_2 \le \frac{C_{\mathrm B}\sigma_T}{S}, \\
\E\bigl[\bigl\lVert R_T(x;\tau,S)-Tx\bigr\rVert_2^2\mid\cF_\tau\bigr] \le \frac{C_{\rm V}\sigma_T^2}{S}, \\
\bigl\lVert \E\bigl[D_T(x,z;\tau,S)-(Tx-Tz)\mid\cF_\tau\bigr]\bigr\rVert_2 \le \frac{C_{\mathrm B}L_T\norm{x-z}}{S}, \\
\E\bigl[\bigl\lVert D_T(x,z;\tau,S)-(Tx-Tz)\bigr\rVert_2^2 \mid \cF_\tau \bigr] \le \frac{C_{\mathrm V}L_T^2\norm{x-z}^2}{S}.
\end{gathered}
\end{equation}
\end{lemma}

\begin{proof}[Proof of \cref{lem:page-blocks-banach}]
After conditioning on $\cF_\tau$, the query points and block length are deterministic. For the refresh estimator $R_T$, we apply \cref{lem:poisson-block} to $g_x(y) := H(x, y) - Tx$. For the difference estimator $D_T$, we apply the same lemma to
\begin{equation}
g_{x, z}(y) := H(x, y) - H(z, y) - (Tx - Tz).
\end{equation}
Both functions have stationary-mean zero. \Cref{cond:page-banach} bounds their $L^2(\pi)$ norms by $\sigma_T$ and $L_T\norm{x - z}$, respectively. Substituting these into \cref{lem:poisson-block} gives the stated inequalities.
\end{proof}

\begin{remark}[Norm-comparison constants]\label{rem:norm-comparison}
The norm-comparison cost factor $\mu_{\norm{\cdot}}$ appears only when Euclidean estimator estimates are converted into the working norm. In the proof of \cref{thm:page-halpern-banach}, it appears in the terminal perturbation bound \eqref{eq:banach-terminal-error} and squared in the increment estimate \eqref{eq:banach-increment-bound}. For many geometries, this comparison factor is mild. For instance, for $\ell_\infty$ and block-sup norms, $\mu_{\norm{\cdot}} \le 1$. For the span seminorm, $\norm{z}_{\mathrm{sp}} \le 2\norm{z}_\infty \le 2\norm{z}_2$, so $\mu_{\norm{\cdot}}\le 2$ before quotient normalization.
\end{remark}
\section{Supporting results for high-probability guarantees}\label{app:hp-details}
\subsection{High-probability Markovian blocks}

We first prove the Markovian concentration estimate that replaces the second-moment block estimate used in \cref{sec:hilbert,sec:banach}.

\begin{lemma}[Smooth-norm Markovian block tail]\label{lem:hp-poisson-block}
Suppose \cref{ass:ergodic} holds. Let $g: \cY \to \R^d$ satisfy
\begin{equation}
\sum_y \pi(y)g(y) = 0,\qquad \sup_{y \in \cY}\norm{g(y)}_{\mathrm{s}} \le M.
\end{equation}
There exists a constant $C_{\mathrm{hp}}$, depending only on the finite-state chain, such that for any stopping time $\tau$, any integer-valued $\cF_\tau$-measurable block length $S \ge 1$, and any $\eta \in (0, 1)$, conditioned on $\cF_\tau$,
\begin{equation}\label{eq:hp-poisson-block}
\Bigl\lVert \frac1S \sum_{i=1}^S g(Y_{\tau+i}) \Bigr\rVert_{\mathrm{s}} \le C_{\mathrm{hp}} M \left(\sqrt{\frac{\kappa_{\mathrm{s}} + \log(1/\eta)}{S}} +\frac1S\right)
\end{equation}
with probability at least $1-\eta$.
\end{lemma}

\begin{proof}[Proof of \cref{lem:hp-poisson-block}]
Condition on $\cF_\tau$. Since $S$ is $\cF_\tau$-measurable, the block length is fixed before the block is sampled. It is therefore enough to prove the estimate for deterministic $S$.

Let $\nu_g = \sum_{t \ge 0}P^t g$ be the Poisson solution. The finite-state geometric-ergodicity bound gives
\begin{equation}
\sup_{y \in \cY} \norm{\nu_g(y)}_{\mathrm{s}}\le C_\nu M,
\end{equation}
where $C_\nu$ depends only on the chain. Let
\begin{equation}
M_i := \nu_g(Y_{\tau+i}) - (P\nu_g)(Y_{\tau+i-1}), \qquad 1 \le i\le S.
\end{equation}
Then $(M_i)_{i=1}^S$ is a martingale-difference sequence conditioned on $\cF_\tau$, and by the Poisson equation,
\begin{equation}
\sum_{i=1}^S g(Y_{\tau+i}) = \sum_{i=1}^S M_i + (P\nu_g)(Y_\tau) - (P\nu_g)(Y_{\tau+S}).
\end{equation}
The boundary term has $\norm{\cdot}_{\mathrm{s}}$-norm at most $2C_\nu M$. Moreover, $\norm{M_i}_{\mathrm{s}} \le 2C_\nu M$ almost surely. Hence the conditional light-tail hypothesis of the vector-valued Freedman inequality in $\kappa_{\mathrm{s}}$-smooth spaces is satisfied with variance proxy of order $M^2$ \citep[Theorem~B.2]{luo2026unified}. Thus, with conditional probability at least $1-\eta$,
\begin{equation}
\Bigl\lVert\sum_{i=1}^S M_i\Bigr\rVert_{\mathrm{s}} \le C M\sqrt{S(\kappa_{\mathrm{s}} + \log(1/\eta))}.
\end{equation}
Dividing the martingale and boundary bounds by $S$ proves \eqref{eq:hp-poisson-block}.
\end{proof}

\begin{lemma}[High-probability block estimates]\label{lem:hp-page-blocks}
Suppose \cref{ass:ergodic} and \cref{cond:hp-page} hold. Let $\tau$ be the current chain time, let $x, z$ be $\cF_\tau$-measurable query points, and let $S \ge 1$ be an integer-valued $\cF_\tau$-measurable block length. Then, conditioned on $\cF_\tau$, the refresh block satisfies
\begin{equation}
\norm{R_T(x; \tau, S) - Tx}_{\mathrm{s}} \le C_{\mathrm{hp}} \sigma_{\mathrm{s}} \left(\sqrt{\frac{\kappa_{\mathrm{s}} + \log(1/\eta)}{S}} + \frac1S\right)
\end{equation}
with probability at least $1 - \eta$, and the difference block satisfies
\begin{equation}
\norm{D_T(x, z; \tau,S) - (Tx - Tz)}_{\mathrm{s}} \le C_{\mathrm{hp}} L_{\mathrm{s}} \norm{x-z} \left(
\sqrt{\frac{\kappa_{\mathrm{s}} + \log(1/\eta)}{S}} +\frac1S \right)
\end{equation}
with probability at least $1 - \eta$.
\end{lemma}

\begin{proof}[Proof of \cref{lem:hp-page-blocks}]
We simply apply \cref{lem:hp-poisson-block} to
\begin{equation}
g_x(y) := H(x,y) - Tx
\end{equation}
for the refresh block, and to
\begin{equation}
g_{x,z}(y) := H(x,y) - H(z,y) - (Tx - Tz)
\end{equation}
for the difference block. Both functions have stationary mean zero. Writing
$\pi_{\min} := \min_{y\in\cY} \pi(y) > 0$, we have by \cref{cond:hp-page},
\begin{equation}
\sup_{y \in \cY} \norm{g_x(y)}_{\mathrm{s}} \le \pi_{\min}^{-1/2} \sigma_{\mathrm{s}}, \qquad \sup_{y \in \cY} \norm{g_{x,z}(y)}_{\mathrm{s}} \le \pi_{\min}^{-1/2} L_{\mathrm{s}}\norm{x-z}.
\end{equation}
\Cref{lem:hp-poisson-block} therefore provides  the stated estimates after absorbing $\pi_{\min}^{-1/2}$ into the chain-dependent constant $C_{\mathrm{hp}}$.
\end{proof}

\subsection{Proof of Lemma~\ref{lem:hp-page-schedule}}
\begin{proof}[Proof of \cref{lem:hp-page-schedule}]
By \cref{lem:hp-page-blocks}, the block-size schedules and a union bound over the blocks used through time $N$ imply that, with probability at least $1 - \eta/2$, every refresh block satisfies
\begin{equation}\label{eq:hp-refresh-good}
\norm{R_T(x_n; \tau, S_{1,n}) - Tx_n}_{\mathrm{s}} \le \frac{\sqrt{p_n}\delta}{8},
\end{equation}
and every difference block satisfies
\begin{equation}\label{eq:hp-diff-good}
\norm{D_T(x_n, x_{n-1}; \tau, S_{2,n}) - (Tx_n - Tx_{n-1})}_{\mathrm{s}} \le \frac{p_n\delta}{8\ell_N}.
\end{equation}
The probability bounds in \cref{lem:hp-page-blocks} are conditional on the sigma-algebra at the corresponding block start. We obtain the simultaneous event above by applying the tower property successively over the blocks. On this event, the estimator error remains uniformly bounded. At a refresh iteration, we have $\norm{e_n}_{\mathrm{s}} \le \delta/8$. At a difference iteration, we write
\begin{equation}
\zeta_n := D_T(x_n, x_{n-1}; \tau, S_{2,n}) - (Tx_n - Tx_{n-1}).
\end{equation}
The recursive update then implies
\begin{equation}
e_n = e_{n-1} + \zeta_n, \qquad \norm{\zeta_n}_{\mathrm{s}} \le \frac{p_n\delta}{8\ell_N}.
\end{equation}
Since $\sum_{n=1}^N p_n \le 2\ell_N$, the total error accumulated during any sequence of difference iterations is at most $\delta/4$. Together with the refresh-block bound, this proves \eqref{eq:hp-estimator-uniform}.

We next control the scaled increments. For $n \ge 2$, we define
\begin{equation}
\Delta_n := \beta_n e_{n-1} - \beta_{n-1}e_{n-2}.
\end{equation}
If iteration $n - 1$ uses a difference block, then
\begin{equation}
\Delta_n = (\beta_n - \beta_{n-1})e_{n-2} + \beta_n\zeta_{n-1}.
\end{equation}
By \eqref{eq:hp-estimator-uniform}, \eqref{eq:hp-diff-good}, and $\beta_n - \beta_{n-1} = 1/(n(n + 1))$, we have
\begin{equation}
n\norm{\Delta_n}_{\mathrm{s}} \le \frac{C\delta}{n} + \frac{C\delta}{\ell_N}.
\end{equation}
Summing over the recursive iterations gives a contribution of at most $C\delta N$. If iteration $n - 1$ instead uses a refresh block, then by \eqref{eq:hp-estimator-uniform},
\begin{equation}
n\norm{\Delta_n}_{\mathrm{s}} \le Cn\delta.
\end{equation}
We let $\tau_m$ denote the chain time immediately before the PAGE coin is tossed at iteration $m$. We have
\begin{equation}
\E\bigl[\mathbf 1_{\cR_m} \mid \cF_{\tau_m}\bigr] = p_m = \frac{2}{m + 1}, \qquad m \ge 1.
\end{equation}
Therefore, by the scalar Bernstein--Freedman inequality applied to the martingale differences $(m + 1)(\mathbf 1_{\cR_m} - p_m)$, with probability at least $1 - \eta/2$,
\begin{equation}
\sum_{m=1}^{N} (m + 1)\mathbf 1_{\cR_m} \le C N(1 + \log(1/\eta)).
\end{equation}
Combining the difference and refresh contributions proves \eqref{eq:hp-estimator-uniform}. Since $p_m = 2/(m + 1)$, the same estimate also proves \eqref{eq:hp-refresh-count}.
\end{proof}

\subsection{High-probability residual bound}

\begin{lemma}[Pathwise displacement-level residual]\label{lem:pathwise-displacement}
Let $\beta_n = n/(n + 1)$, and consider the recursion
\begin{equation}
x_n = (1 - \beta_n)x_0 + \beta_n(Tx_{n-1} + U_n), \qquad n \ge 1.
\end{equation}
Let $U_0 := 0$ and $\Delta_n := \beta_nU_n - \beta_{n-1}U_{n-1}$. Suppose that, up to horizon $N$,
\begin{equation}
\sup_{0 \le n \le N}\norm{Tx_n - x_0} \le K, \qquad \norm{U_N} \le r, \qquad \sum_{n=1}^N n\norm{\Delta_n} \le A.
\end{equation}
Then,
\begin{equation}
\norm{x_N - Tx_N} \le \frac{K(1 + \log(N + 1))}{N + 1} + \frac{A}{N + 1} + r.
\end{equation}
\end{lemma}

\begin{proof}[Proof of \cref{lem:pathwise-displacement}]
Let $d_n := x_n - x_{n-1}$. Subtracting two consecutive inexact Halpern updates gives
\begin{equation}
d_n = (\beta_n - \beta_{n-1})(Tx_{n-1} - x_0) + \beta_{n-1}(Tx_{n-1} - Tx_{n-2}) + \Delta_n.
\end{equation}
By non-expansiveness of $T$ and the assumed pathwise bound,
\begin{equation}
\norm{d_n} \le \frac{K}{n(n + 1)} + \frac{n - 1}{n}\norm{d_{n-1}} + \norm{\Delta_n}.
\end{equation}
Multiplying by $n$ and summing from $1$ to $N$ yields
\begin{equation}
N\norm{d_N} \le K\sum_{n=1}^N \frac{1}{n + 1} + \sum_{n=1}^N n\norm{\Delta_n} \le K\log(N + 1) + A.
\end{equation}
The final residual satisfies
\begin{equation}
\norm{x_N - Tx_N} \le \frac{K}{N + 1} + \frac{N}{N + 1}\norm{d_N} + r.
\end{equation}
Substituting the displacement bound proves the result.
\end{proof}

\section{Experimental details}\label{app:experiments}
\subsection{Hilbert-space policy evaluation}

We arrange the states $\{0,\ldots,7\}$ on a cycle. The resulting process is a lazy nearest-neighbor random walk. At each step, the state remains unchanged with probability $1 - 2q$ and moves to either neighboring state with probability $q$:
\begin{equation}
P(s,s) = 1 - 2q, \qquad P(s,s - 1 \bmod 8) = P(s,s + 1 \bmod 8) = q.
\end{equation}
Smaller values of $q$ make the trajectory remain near its current state for longer. Since the transition matrix is symmetric, its stationary distribution is uniform. We choose the alternating differential value function $\phi(s) = (-1)^s$. The transition reward is
\begin{equation}
r(s,s') = \phi(s) - \phi(s') + 0.2\bigl(\mathbf 1\{s' = s + 1 \bmod 8\} - \mathbf 1\{s' = s - 1 \bmod 8\}\bigr).
\end{equation}
The first term is a potential difference. Along any realized trajectory, these differences telescope according to
\begin{equation}
\sum_{t=0}^{m-1}\bigl(\phi(S_t) - \phi(S_{t+1})\bigr) = \phi(S_0) - \phi(S_m).
\end{equation}
Thus, the relative return associated with the initial state is described by $\phi$. The second term does not alter this mean behavior. It is $0.2$ for a clockwise transition, $-0.2$ for a counterclockwise transition, and zero for a self-transition. The clockwise and counterclockwise moves have equal probability, so this term has conditional mean zero. We include it to ensure that individual transition samples remain noisy even at the exact solution. More precisely, the mean reward satisfies
\begin{equation}
\bar r(s) := \E\bigl[r(S_t,S_{t+1}) \mid S_t = s\bigr] = \phi(s) - (P\phi)(s).
\end{equation}
Let $\pi$ denote the uniform stationary distribution. Since $\pi^\top P = \pi^\top$, we have
\begin{equation}
\sum_s \pi(s)\bar r(s) = \pi^\top(\phi - P\phi) = \pi^\top\phi - (\pi^\top P)\phi = 0.
\end{equation}
Equivalently, under stationarity, $S_t$ and $S_{t+1}$ have the same marginal distribution, so the potential difference has mean zero. Thus, the average reward is zero. The differential Bellman equation is \citep{tsitsiklis1999average,wan2021learning}
\begin{equation}
h = \bar r + Ph.
\end{equation}
Substitution shows directly that $h^* = \phi$ is a solution, since $\bar r + P\phi = \phi$. We define the Bellman operator by $Bh := \bar r + Ph$. The matrix $P$ averages the values at the current state and its two neighbors. Since this averaging matrix is symmetric, all of its eigenvalues lie in $[-1,1]$. Consequently,
\begin{equation}
\norm{Bh - Bz}_2 = \norm{P(h-z)}_2 \le \norm{h-z}_2,
\end{equation}
so this problem falls within the Hilbert-space theory of \cref{sec:hilbert}. The algorithm observes the continuing transition process $Y_t = (S_t,S_{t+1})$. For a candidate value vector $h$, the temporal-difference error associated with one transition is
\begin{equation}
\delta_h(s,s') := r(s,s') + h(s') - h(s).
\end{equation}
For $e_s$ denoting the $s$th coordinate vector, we define the single-transition operator evaluation
\begin{equation}\label{eq:experimental-oracle}
H_{\mathrm{eval}}(h,(s,s')) = h + 8e_s\delta_h(s,s').
\end{equation}
This operator samples the Bellman equation only at the observed state $s$. The multiplier eight is an importance weight rather than an algorithmic stepsize. Each state appears with stationary probability $1/8$, so the multiplier converts the randomly selected coordinate correction into the full Bellman correction in expectation. Indeed,
\begin{equation}
\E_\pi\left[8e_{S_t}\delta_h(S_t,S_{t+1})\right] = \bar r + Ph - h,
\end{equation}
and hence $\E_\pi[H_{\mathrm{eval}}(h,Y)] = Bh$. Once a transition $(s,s')$ has been observed, the expression in \eqref{eq:experimental-oracle} can be evaluated at two query points. The same-state difference oracle required by PAGE is therefore available without observing an additional transition.

\subsection{Banach-space average-reward control}

The second example retains the state process and transition noise from the policy-evaluation problem. We introduce two actions $\cA = \{0,1\}$, but the selected action does not affect the transition matrix. The behavior policy selects each action with probability $1/2$, so the stationary state-action distribution is uniform. Action one incurs a penalty $\Delta = 0.5$, and the reward is
\begin{equation}
r(s,a,s') = \phi(s) - \phi(s') - \Delta a + 0.2\bigl(\mathbf 1\{s' = s + 1 \bmod 8\} - \mathbf 1\{s' = s - 1 \bmod 8\}\bigr).
\end{equation}
Action zero is therefore optimal. The control decision is deliberately simple. The purpose of introducing the actions is to place the maximization operator inside the Bellman map while keeping the Markovian sampling problem unchanged. The mean reward is $\bar r(s,a) = \phi(s) - (P\phi)(s) - \Delta a$, the average-reward optimality operator is
\begin{equation}\label{eq:control-bellman}
(\mathcal TQ)(s,a) = \bar r(s,a) + \sum_{s'}P(s,s')\max_{b \in \cA}Q(s',b),
\end{equation}
and the optimal average reward is zero. We have the fixed point
\begin{equation}
Q^*(s,a) = \phi(s) - \Delta a,
\end{equation}
since $\max_b Q^*(s',b) = Q^*(s',0) = \phi(s')$. Therefore,
\begin{equation}
(\mathcal TQ^*)(s,a) = \phi(s) - (P\phi)(s) - \Delta a + (P\phi)(s) = Q^*(s,a).
\end{equation}
Although the optimal action is simple, the operator $\mathcal T$ is globally nonlinear because the maximizing action changes with the input $Q$. For any $Q,Z \in \R^{16}$, we have
\begin{equation}
\begin{aligned}
\left\lvert\max_b Q(s,b) - \max_b Z(s,b)\right\rvert
&\le \max_b\left\lvert Q(s,b) - Z(s,b)\right\rvert \le \norm{Q - Z}_\infty,\\
\norm{\mathcal TQ - \mathcal TZ}_\infty
&\le \max_s \sum_{s'}P(s,s')\norm{Q - Z}_\infty = \norm{Q - Z}_\infty.
\end{aligned}
\end{equation}
Thus, $\mathcal T$ is non-expansive in the sup norm, whereas no norm induced by an inner product makes $\mathcal T$ non-expansive. The observed process is now $Y_t = (S_t,A_t,S_{t+1})$. For a candidate action-value vector $Q$, we define the sampled optimality error by
\begin{equation}
\delta_Q(s,a,s') := r(s,a,s') + \max_{b \in \cA}Q(s',b) - Q(s,a).
\end{equation}
For $e_{(s,a)}$ denoting the corresponding coordinate vector in $\R^{16}$, the single-transition evaluation is
\begin{equation}\label{eq:experimental-control-oracle}
H_{\mathrm{ctl}}(Q,(s,a,s')) = Q + 16e_{(s,a)}\delta_Q(s,a,s').
\end{equation}
The multiplier sixteen is again an importance weight. Each of the sixteen state-action pairs has stationary probability $1/16$, and hence $\E_\pi[H_{\mathrm{ctl}}(Q,Y)] = \mathcal TQ$. Once $(s,a,s')$ has been observed, the sampled optimality error can be evaluated at two different vectors $Q$ and $Z$ without another environment transition. The PAGE difference evaluation is therefore available in the same online trajectory.

\subsection{Methods and evaluation}

We use $q = 0.24$ for the fast-mixing chain and $q = 0.04$ for the slow-mixing chain. The corresponding second-largest eigenvalue moduli are $0.859$ and $0.977$. We initialize the policy-evaluation iterate at $h_0 = 0$ and the control iterate at $Q_0 = 0$. We generate twenty independent trajectories for each problem and each chain. Within every repetition, the block and PAGE methods use the same realized trajectory. Each Markov chain is initialized at $S_0 = 0$ rather than from its stationary distribution.

For the block method, we use $k_n = \lceil 0.01n^4\rceil$ and the burn-in schedule $b_n = \lceil 2\log(n+1)/\log(1/\rho)\rceil$ with the exact eigenvalue modulus in place of $\rho$. This optional burn-in reduces finite-horizon sensitivity to the initial state and block boundaries. For both PAGE methods, we use $p_n = \min\{1,2/(n+1)\}$ and the fixed estimator scale $\delta = 0.02$. The problem-dependent constants in the PAGE block schedules are replaced by a common multiplier $c = 0.15$. Thus,

\begin{equation}
S_{1,0} = \left\lceil \frac{c}{\delta^2}\right\rceil, \qquad S_{1,n} = \left\lceil \frac{c}{p_n\delta^2}\right\rceil,
\end{equation}
and
\begin{equation}
S_{2,n} = \left\lceil \max\left\{\frac{cd_n^2}{p_n^2\delta^2},\frac{cd_n}{p_n^{3/2}\delta},1\right\}\right\rceil.
\end{equation}
For policy evaluation, $d_n = \norm{h_n - h_{n-1}}_2$, and the Hilbert method uses $L^* = 2$. For control, $d_n = \norm{Q_n - Q_{n-1}}_\infty$, and the Banach method applies the estimator of $\mathcal TQ_n$ directly. The same parameters are fixed for both chains and all transition budgets.

We evaluate all methods at the common budgets
\begin{equation}
B_j = 500 \cdot 2^j, \qquad j = 0,\ldots,12.
\end{equation}
At budget $B_j$, we report the last iterate available after at most $B_j$ transitions. If the method is partway through a block at that time, the incomplete block is not used and the preceding iterate is reported. Each method consumes one uninterrupted trajectory. Every transition is included in the budget, including the transitions discarded during the optional burn-in periods of the baseline block method. A PAGE difference block evaluates the oracle at two query points for each observed transition. We use the number of environment transitions as the primary cost. For policy evaluation, we report
\begin{equation}
\frac{\norm{h - \bar r - Ph}_2}{\norm{\bar r}_2}.
\end{equation}
For control, we report
\begin{equation}
\frac{\norm{Q - \mathcal TQ}_\infty}{\norm{\bar r}_\infty}.
\end{equation}
Because a small number of trajectories can produce unusually large residuals, \cref{fig:average-reward} reports the median residual at each common budget and shades the interquartile range from the 25th to the 75th percentile. We also record the first completed iterate whose residual is at most $\epsilon$ for $\epsilon \in \{0.20,0.10,0.05\}$. These transition counts are reported in \cref{tab:average-reward-hitting}. If a run does not reach a target within $B_{12} = 2{,}048{,}000$ transitions, we record it as unsuccessful at that target.

\begin{table}[t]
\centering
\small
\setlength{\tabcolsep}{4pt}
\caption{Median number of observed transitions required to reach a common residual target. The counts in parentheses give the number of repetitions that reach the target when fewer than twenty do so. The notation $>2{,}048{,}000$ means that the median transition count exceeds the available budget.}
\label{tab:average-reward-hitting}
\begin{tabular}{lllrrr}
\toprule
Problem & Mixing & Method & $\epsilon = 0.20$ & $\epsilon = 0.10$ & $\epsilon = 0.05$ \\
\midrule
\multirow{4}{*}{Evaluation}
& \multirow{2}{*}{Fast}
& Block Halpern & $585$ & $1{,}238$ & $8{,}839$ \\
& & Hilbert PAGE-Halpern & $4{,}567$ & $9{,}001$ & $14{,}684$ \\
\cmidrule(lr){2-6}
& \multirow{2}{*}{Slow}
& Block Halpern & $68{,}885$ & $1{,}923{,}959$ & $>2{,}048{,}000\ (0/20)$ \\
& & Hilbert PAGE-Halpern & $25{,}705$ & $62{,}611$ & $107{,}374$ \\
\midrule
\multirow{4}{*}{Control}
& \multirow{2}{*}{Fast}
& Block Halpern & $6{,}205$ & $21{,}021$ & $63{,}081$ \\
& & Banach PAGE-Halpern & $3{,}302$ & $8{,}606$ & $20{,}749$ \\
\cmidrule(lr){2-6}
& \multirow{2}{*}{Slow}
& Block Halpern & $51{,}481\ (19/20)$ & $289{,}018\ (17/20)$ & $>2{,}048{,}000\ (8/20)$ \\
& & Banach PAGE-Halpern & $13{,}430$ & $46{,}415$ & $80{,}000$ \\
\bottomrule
\end{tabular}
\end{table}

\section{Proof of the displacement-level bound}
\subsection{Proof of Theorem~\ref{thm:displacement-level}}
\begin{proof}[Proof of \cref{thm:displacement-level}]
Let $d_n := x_n - x_{n-1}$. By the Halpern recursion,
\begin{equation}
d_n = (\beta_n - \beta_{n-1})(Tx_{n-1}-x_0)+\beta_{n-1}(T x_{n-1} - Tx_{n-2}) + \Delta_n,
\end{equation}
(with the second term omitted when $n=1$.) By the step size schedule, we also have
\begin{equation}
\beta_n - \beta_{n-1} = \frac{1}{n(n+1)}, \qquad \beta_{n-1}=\frac{n-1}{n}.
\end{equation}
By \cref{ass:oracle-stable} and using non-expansiveness,
\begin{equation}
\E\bigl[\norm{d_n}\bigr] \le \frac{\kappa_1}{n(n+1)} + \frac{n-1}{n}\E\bigl[\norm{d_{n-1}}\bigr] + a_n.
\end{equation}
Multiplying both sides by $n$ results in the telescoping recursion
\begin{equation}
n\E\bigl[\norm{d_n}\bigr] \le \frac{\kappa_1}{(n+1)} + (n-1)\E\bigl[\norm{d_{n-1}}\bigr] + na_n.
\end{equation}
Summing both sides from $1$ to $N$:
\begin{equation}\label{eq:displacement}
N \E\bigl[\norm{d_N}\bigr] \le \kappa_1 \sum_{n=1}^N \frac{1}{n+1} + \sum_{n=1}^N n a_n \le \kappa_1\log(N+1) + \sum_{n=1}^N n a_n.
\end{equation}
Now, the final iterate residual is
\begin{equation}
x_N - T x_N = (1-\beta_N)(x_0 - T x_N) + \beta_N(T x_{N-1} - Tx_N) + \beta_N U_N.
\end{equation}
By non-expansiveness and \cref{ass:oracle-stable},
\begin{equation}
\E\bigl[\norm{x_N - T x_N}\bigr] \le \frac{\kappa_1}{N+1} + \frac{N}{N+1}\E\bigl[\norm{d_N}\bigr] + r.
\end{equation}
Using \eqref{eq:displacement} and increasing $1+\log(N+1)$ to absorb the extra $\kappa_1/(N+1)$ term yields the first claim. The remaining statements follow by Jensen's inequality.
\end{proof}

\section{Proofs of the main Hilbert-space results}
\subsection{Proof of Lemma~\ref{lem:page-schedule}}
\begin{proof}[Proof of \cref{lem:page-schedule}]
Let $e_n := v_n-F(x_n)$ and $a_n := \E\bigl[\norm{e_n}_2^2\bigr]$. The proof is an induction on the estimator error. The initialization uses a refresh block whose size satisfies the first inequality in \eqref{eq:refresh-block-size}, so \cref{lem:page-blocks} directly gives
\begin{equation}
a_0 = \E\bigl[\norm{v_0-F(x_0)}_2^2\bigr] \le \delta^2.
\end{equation}
Now, fix $n \ge 1$ and assume the induction hypothesis
\begin{equation}
a_{n-1} \le \frac{100\delta^2}{n}.
\end{equation}
At the moment of the PAGE coin toss, $x_n$, $x_{n-1}$, $v_{n-1}$, and therefore $e_{n-1}$, are all measurable with respect to the past filtration. The next block is sampled only after the branch and block length have been chosen. By the refresh-block bound of \cref{lem:page-blocks} and the second inequality in \eqref{eq:refresh-block-size},
\begin{equation}
\E\bigl[\norm{R(x_n;\tau,S_{1,n})-F(x_n)}_2^2\bigr] \le p_n\delta^2.
\end{equation}
This is the contribution to $a_n$ in case that a refresh block is used at iteration $n$, the probability of which is $p_n$. In the case a difference block is used, we write
\begin{equation}
\zeta_n := D(x_n, x_{n-1}; \tau, S_{2,n}) - (F(x_n) - F(x_{n-1})).
\end{equation}
By \cref{lem:page-blocks} and the third inequality in \eqref{eq:refresh-block-size},
\begin{equation}
\E\bigl[\norm{\zeta_n}_2^2 \mid \cF_\tau] \le p_n^2\delta^2, \qquad \bigl\lVert\E\bigl[\zeta_n \mid \cF_\tau\bigr]\bigr\rVert_2 \le p_n^{3/2}\delta.
\end{equation}
In this case, we have by the update rule in \eqref{eq:page-halpern-1}--\eqref{eq:page-halpern-2} that
\begin{equation}
v_n = v_{n-1} + D(x_n, x_{n-1}; \tau, S_{2,n}),
\end{equation}
so the estimator error evolves as
\begin{equation}
e_n = v_n - F(x_n) = \bigl(v_{n-1} + D(x_n, x_{n-1}; \tau, S_{2,n})\bigr) - F(x_n) + \bigl(F(x_{n-1}) - F(x_{n-1})\bigr) = e_{n-1} + \zeta_n.
\end{equation}
Since $e_{n-1}$ is $\cF_\tau$-measurable, we have
\begin{equation}
\begin{aligned}
\E\bigl[\bigl\lVert e_{n-1}+\zeta_n\bigr\rVert_2^2\bigr] &= a_{n-1} + 2\E\bigl[\langle e_{n-1},\E[\zeta_n\mid\cF_\tau]\rangle\bigr] + \E\bigl[\lVert \zeta_n\rVert_2^2\bigr] \\
&\le a_{n-1} + 2\sqrt{a_{n-1}}\,p_n^{3/2}\delta + p_n^2\delta^2 \\
&\le a_{n-1} + 21p_n^2\delta^2,
\end{aligned}
\end{equation}
where the last inequality uses the induction hypothesis and $p_n^{3/2}/\sqrt n \le p_n^2$, which follows from $p_n = 2/(n+1)$ for $n \ge 1$. Thus the difference branch has conditional second moment at most $a_{n-1} + 21p_n^2\delta^2$. Combining the bounds for the refresh and difference branches gives
\begin{equation}
a_n \le p_n^2\delta^2 + (1 - p_n)\left(\frac{100\delta^2}{n}+21p_n^2\delta^2\right).
\end{equation}
For $n \ge 1$, $p_n=2/(n+1)$, hence
\begin{equation}\label{eq:absorber}
(1-p_n)\frac{100\delta^2}{n} = \frac{100(n-1)\delta^2}{n(n+1)} = \frac{100\delta^2}{n+1} - \frac{100\delta^2}{n(n+1)}.
\end{equation}
The remaining terms satisfy $p_n^2\delta^2 + (1-p_n)21p_n^2\delta^2 \le 22p_n^2\delta^2 \le
88\delta^2/(n(n+1))$, and this is absorbed by the negative term $-100\delta^2/(n(n+1))$ in \eqref{eq:absorber}. Therefore
\begin{equation}
a_n \le \frac{100\delta^2}{n+1}.
\end{equation}
This proves the induction.
\end{proof}

\subsection{Proof of Proposition~\ref{prop:cocoercive-page}}
\begin{proof}[Proof of \cref{prop:cocoercive-page}]
We state a known result on cocoercive monotone inclusions in the notation of the present paper for completeness, see \citet{cai2022stochastic}. Under the stated estimator accuracy condition, the PAGE--Halpern iterates satisfy \citep[Theorem 3.1]{cai2022stochastic}
\begin{equation}
\E\bigl[\norm{F(x_N)}_2\bigr] \le \frac{C_0 L^* R}{N} + C_1\delta
\end{equation}
for universal constants $C_0, C_1$. Although the cited theorem states unbiasedness at initialization, its proof uses this property only to remove the initial mixed error term. Bounding that term by Cauchy's inequality under the displayed second-moment condition changes only the universal constant $C_1$. Furthermore, with $p_n=2/(n+1)$ and $\delta N \le c_2L^*R$, the weighted displacement accumulated by the PAGE recursion satisfies \citep[Lemma 3.3]{cai2022stochastic}
\begin{equation}
\sum_{n=1}^N \frac{\E\bigl[\norm{x_n-x_{n-1}}_2^2\bigr]}{p_n^2} \le C_2 R^2 N
\end{equation}
for a universal constant $C_2$.
\end{proof}

\subsection{Proof of Theorem~\ref{thm:page-halpern}}
\begin{proof}[Proof of \cref{thm:page-halpern}]
The claim is immediate in the two trivial cases identified in the theorem, so we assume that $R > 0$ and $\epsilon < 2R$. The operator $F = I - T$ is monotone and $1/2$-cocoercive, as shown at the start of this section.  Since $L^* \ge 2$, it is also $1/L^*$-cocoercive. Thus, the conditions required by \cref{prop:cocoercive-page} are satisfied. \Cref{lem:page-schedule} guarantees $\E\bigl[\norm{v_n-F(x_n)}_2^2\bigr] \le 100\delta^2/(n+1)$, so the residual conclusion of \cref{prop:cocoercive-page} applies with estimator scale $10\delta$. Hence
\begin{equation}
\E\bigl[\norm{F(x_N)}_2\bigr] \le \frac{C_0 L^* R}{N} + 10 C_1\delta \le \epsilon,
\end{equation}
since the choices of $N$ and $\delta$ make each of the two terms at most $\epsilon/2$. Because $F(x_N) = x_N - Tx_N$, this is exactly the fixed-point residual bound. For the sample-complexity calculation, we use the displacement conclusion of \cref{prop:cocoercive-page}. We may enlarge the universal constant $C_1$ so that $C_1 \ge (2C_0+1)/(2c_2)$. Since $L^* \ge 2$ and $\epsilon<2R$, the choices of $N$ and $\delta$ then give
\begin{equation}
10\delta N \le \frac{2C_0L^*R+\epsilon}{2C_1} \le c_2L^*R.
\end{equation}
Thus, the displacement bound in \cref{prop:cocoercive-page} applies with estimator scale $10\delta$. It remains to bound the sample complexity.  We count expected Markov transitions along the single continuing chain.  At time $n$, a refresh block is used with probability $p_n$, while its prescribed length is
\begin{equation}
S_{1,n}=O\left(\frac{C_{\mathrm V}\sigma_F^2}{p_n\delta^2}\right).
\end{equation}
Therefore the refresh cost satisfies
\begin{equation}
\sum_{n=1}^N p_n S_{1,n} = O\left(\frac{C_{\mathrm V}\sigma_F^2 N}{\delta^2}\right).
\end{equation}
The unit term in the third inequality of \eqref{eq:refresh-block-size} contributes at most $O(N)$, which is lower order compared with the other two accuracy-dependent terms. For the first term in the maximum, the variance-control part of difference batches, the probability of occurrence is at most one, so it is enough to sum the prescribed variance-control lengths.  Using
\eqref{eq:halpern-displacement},
\begin{equation}
\sum_{n=1}^N \frac{C_{\mathrm V} L_F^2}{\delta^2} \frac{\E\bigl[\norm{x_n-x_{n-1}}_2^2\bigr]}{p_n^2} = O\left(\frac{C_{\mathrm V} L_F^2 R^2 N}{\delta^2}\right).
\end{equation}
The second term in the maximum in the third inequality of \eqref{eq:refresh-block-size}, which controls the bias, contributes
\begin{equation}
O\left(\frac{C_{\mathrm B} L_F}{\delta} \sum_{n=1}^N \frac{\E\bigl[\norm{x_n-x_{n-1}}_2\bigr]}{p_n^{3/2}}\right).
\end{equation}
To bound the remaining sum, apply Cauchy's inequality:
\begin{equation}
\begin{aligned}
\sum_{n=1}^N \frac{\E\bigl[\norm{x_n-x_{n-1}}_2\bigr]}{p_n^{3/2}} &\leq \left(\sum_{n=1}^N \frac{\E\bigl[\norm{x_n-x_{n-1}}_2^2\bigr]}{p_n^2}\right)^{1/2} \left(\sum_{n=1}^N\frac{1}{p_n}\right)^{1/2} \\
&= O(R N^{3/2}).
\end{aligned}
\end{equation}
The initialization $S_{1,0} = O(C_{\mathrm V}\sigma_F^2/\delta^2)$ is dominated by the refresh sum when $N \ge 1$. Substituting $N = O(L^* R/\epsilon)$, $\delta = \Theta(\epsilon)$, and $L_F \le L^*$, the refresh term contributes
\begin{equation}
O\left(\frac{C_{\mathrm V} \sigma_F^2 L^* R}{\epsilon^3}\right),
\end{equation}
the variance-control difference term contributes
\begin{equation}
O\left(\frac{C_{\mathrm V} (L^*R)^3}{\epsilon^3}\right),
\end{equation}
and the bias-control difference term contributes
\begin{equation}
O\left(\frac{C_{\mathrm B} (L^*R)^{5/2}}{\epsilon^{5/2}}\right).
\end{equation}
Combining these contributions gives \eqref{eq:page-complexity}.
\end{proof}

\section{Proofs of the main Banach-space results}

\subsection{Proof of Lemma~\ref{lem:page-schedule-banach}}
\begin{proof}[Proof of \cref{lem:page-schedule-banach}]
Let $e_n := v_n - Tx_n$ and $a_n := \bE\bigl[\norm{e_n}^2_2\bigr]$. The proof is identical in structure to \cref{lem:page-schedule}, with $T$ replacing $F$. $a_0 \le \delta^2$ follows from the refresh block bound in \cref{lem:page-blocks-banach} and the first inequality in \eqref{eq:refresh-block-size-banach}. Now, assume $a_{n-1} \le 100 \delta^2/n$. In the case of using a refresh block at iteration $n$, the second inequality in \eqref{eq:refresh-block-size-banach} gives
\begin{equation}
\E\bigl[\norm{R_T(x_n; \tau, S_{1, n}) - T x_n}^2_2] \le p_n \delta^2.
\end{equation}
In the case of using a difference block, write
\begin{equation}
\zeta_n := D_T(x_n, x_{n-1}; \tau, S_{2, n}) - (Tx_n-Tx_{n-1}).
\end{equation}
By \cref{lem:page-blocks-banach} and the third inequality in \eqref{eq:refresh-block-size-banach},
\begin{equation}
\E\bigl[\norm{\zeta_n}^2_2 \mid \cF_\tau\bigr] \le p^2_n \delta^2, \qquad \bigl\lVert\E[\zeta_n \mid \cF_\tau]\bigr\rVert_2 \le p_n^{3/2} \delta
\end{equation}
In this branch, $e_n = e_{n-1} + \zeta_n$. Since $e_{n-1}$ is $\cF_\tau$-measurable,
\begin{equation}
\begin{aligned}
\E\bigl[\norm{e_{n-1} + \zeta_n}^2_2\bigr] &= a_{n-1} + 2\E\bigl[\langle e_{n-1}, \E[\zeta_n \mid \cF_\tau]\rangle\bigr] + \E \bigl[\norm{\zeta_n}^2_2\bigr] \\ 
&\le a_{n-1} + 2\sqrt{a_{n-1}}p_n^{3/2}\delta + p_n^2 \delta^2 \\
&\le a_{n-1} + 21 p_n^2 \delta^2.
\end{aligned}
\end{equation}
Combining the bounds for the two blocks,
\begin{equation}
a_n \le p_n^2 \delta^2 + (1-p_n) \left(\frac{100\delta^2}{n} + 21 p_n^2\delta^2\right).
\end{equation}
Finally, the same algebra as in the proof of \cref{lem:page-schedule} yields $a_n \le 100\delta^2 / (n+1)$.
\end{proof}

\subsection{Proof of Theorem~\ref{thm:page-halpern-banach}}
\begin{proof}[Proof of \cref{thm:page-halpern-banach}]
The proof will follow through establishing bounds on the final perturbation $\E\bigl[\norm{U_N}\bigr]$ and on the scaled perturbation increments $\Delta_n$ in \cref{thm:displacement-level}. Let $e_n:=v_n-Tx_n$. We start by writing
\begin{equation}
x_n = (1-\beta_n)x_0 + \beta_n v_{n-1} = (1-\beta_n)x_0 + \beta_n(Tx_{n-1} + e_{n-1}).
\end{equation}
Thus \eqref{eq:page-halpern-banach-1} is exactly \eqref{eq:recursion-with-error} with $U_n=e_{n-1}$. \Cref{lem:page-schedule-banach} provides the Euclidean estimator bound \eqref{eq:banach-estimator-level}. Converting this estimate to the non-expansiveness norm gives us a bound on the final perturbation norm:
\begin{equation}\label{eq:banach-terminal-error}
\E\bigl[\norm{U_N}\bigr] = \E\bigl[\norm{e_{N-1}}\bigr] \le \mu_{\norm{\cdot}}\E\bigl[\norm{e_{N-1}}_2\bigr] \le \mu_{\norm{\cdot}}\bigl(\E\bigl[\norm{e_{N-1}}_2^2\bigr]\bigr)^{1/2} \le C\mu_{\norm{\cdot}}\delta.
\end{equation}
It remains to control the scaled perturbation increments
\begin{equation}
\Delta_n = \beta_n U_n - \beta_{n-1}U_{n-1} = \beta_n e_{n-1} - \beta_{n-1} e_{n-2},
\end{equation}
with the convention $e_{-1}=0$. For $n=1$, we have $\Delta_1 = \beta_1 e_0$, and \eqref{eq:banach-estimator-level} implies $\E\bigl[\norm{\Delta_1}_2^2\bigr]\le C\delta^2$. Now, fix $n \ge 2$, and let $\tau$ be the current chain time immediately before the PAGE coin is tossed at iteration $n-1$. The variables $x_{n-1}$, $x_{n-2}$, $v_{n-2}$, and hence $e_{n-2}$, are $\cF_\tau$-measurable. Let $\cR_{n-1}$ and $\cD_{n-1}$ denote the events that the coin selects a refresh block and a difference block, respectively. Since the PAGE coin is independent of $\cF_\tau$,
\begin{equation}
\PP(\cR_{n-1}\mid\cF_\tau) = p_{n-1},\qquad \PP(\cD_{n-1}\mid\cF_\tau) = 1-p_{n-1}.
\end{equation}
On the event $\cD_{n-1}$, we define
\begin{equation}
\zeta_{n-1} := D_T(x_{n-1}, x_{n-2}; \tau, S_{2,n-1})-(Tx_{n-1} - Tx_{n-2}),
\end{equation}
so that $e_{n-1} = e_{n-2} + \zeta_{n-1}$. Therefore,
\begin{equation}
\Delta_n = (\beta_n - \beta_{n-1}) e_{n-2} + \beta_n \zeta_{n-1}.
\end{equation}
Using $\norm{a+b}_2^2 \le 2\norm{a}_2^2 + 2\norm{b}_2^2$, the estimator bound \eqref{eq:banach-estimator-level}, and the difference-block bound in \cref{lem:page-schedule-banach}, the contribution of $\cD_{n-1}$ to the second moment of $\Delta_n$ satisfies
\begin{equation}
\begin{aligned}
\E[\norm{\Delta_n}_2^2\mathbf 1_{\cD_{n-1}}] &\le 2(\beta_n - \beta_{n-1})^2 \E[\norm{e_{n-2}}_2^2\mathbf 1_{\cD_{n-1}}] + 2\beta_n^2\E[\norm{\zeta_{n-1}}_2^2\mathbf 1_{\cD_{n-1}}] \\
&\le 2(\beta_n - \beta_{n-1})^2\E[\norm{e_{n-2}}_2^2] + 2\beta_n^2p_{n-1}^2\delta^2 \\
&\le C\left((\beta_n - \beta_{n-1})^2\frac{\delta^2}{n} + p_{n-1}^2\delta^2\right) \le \frac{C\delta^2}{n^2}.
\end{aligned}
\end{equation}
Here we use $\beta_n - \beta_{n-1} = 1/(n(n+1))$, $\beta_n \le 1$, and $p_{n-1} = O(1/n)$.
On the event $\cR_{n-1}$, by the refresh block schedule
\begin{equation}
\E\bigl[\norm{e_{n-1}}_2^2 \mid \cF_\tau, \cR_{n-1}\bigr] \le p_{n-1}\delta^2,
\end{equation}
while by \eqref{eq:banach-estimator-level},
\begin{equation}
\E\big[\norm{e_{n-2}}_2^2]\le \frac{C\delta^2}{n}.
\end{equation}
Using $\E\bigl[\mathbf 1_{\cR_{n-1}} \mid \cF_\tau\bigr] = p_{n-1}$ and $\norm{a - b}_2^2 \le 2\norm{a}_2^2 + 2\norm{b}_2^2$,
\begin{equation}
\begin{aligned}
\E\bigl[\norm{\Delta_n}_2^2\mathbf 1_{\cR_{n-1}}\bigr] &\le C\E\bigl[\norm{e_{n-1}}_2^2\mathbf 1_{\cR_{n-1}}\bigr] + C\E\bigl[\norm{e_{n-2}}_2^2\mathbf 1_{\cR_{n-1}}\bigr] \\
&\le Cp_{n-1}\left(p_{n-1}\delta^2+\frac{\delta^2}{n}\right) \\
&\le \frac{C\delta^2}{n^2},
\end{aligned}
\end{equation}
where we use $p_{n-1}=O(1/n)$ in the last inequality. Since $\cR_{n-1}$ and $\cD_{n-1}$ constitute a partition of the sample space, combining their contributions yields
\begin{equation}
\E\bigl[\norm{\Delta_n}_2^2\bigr] \le \frac{C\delta^2}{n^2},\qquad n \ge 1.
\end{equation}
By the norm comparison in \cref{sec:prelim},
\begin{equation}\label{eq:banach-increment-bound}
\E\bigl[\norm{\Delta_n}^2\bigr] \le \mu_{\norm{\cdot}}^2 \E\bigl[\norm{\Delta_n}_2^2\bigr] \le \frac{C\mu_{\norm{\cdot}}^2\delta^2}{n^2}.
\end{equation}
Applying \cref{thm:displacement-level} with \eqref{eq:banach-terminal-error} and \eqref{eq:banach-increment-bound} proves the residual bound statement \eqref{eq:banach-page-residual}. We now establish the sample complexity result. A refresh block is used with probability $p_n$ and has length
\begin{equation}
S_{1,n} = O\left(\frac{C_{\mathrm V}\sigma_T^2}{p_n\delta^2}\right),
\end{equation}
so the expected refresh cost through time $N$ is
\begin{equation}\label{eq:banach-refresh-cost}
O\left(\frac{C_{\mathrm V} \sigma_T^2N}{\delta^2}\right).
\end{equation}
For the difference blocks, we need a displacement estimate in $\norm{\cdot}$. Let
\begin{equation}
d_n := x_n-x_{n-1},\qquad D_n := \bigl(\E\bigl[\norm{d_n}^2\bigr]\bigr)^{1/2}.
\end{equation}
From the inexact Halpern recursion,
\begin{equation}
d_n = (\beta_n - \beta_{n-1})(Tx_{n-1} - x_0) + \beta_{n-1}(Tx_{n-1} - Tx_{n-2}) + \Delta_n.
\end{equation}
Taking $L^2$ norms, using Minkowski's inequality, non-expansiveness of $T$, \cref{lem:stability-constants}, and \eqref{eq:banach-increment-bound}, we obtain
\begin{equation}
D_n \le \frac{\kappa}{n(n+1)} + \frac{n-1}{n} D_{n-1} + \frac{C\mu_{\norm{\cdot}}\delta}{n}.
\end{equation}
Multiplying by $n$,
\begin{equation}
n D_n \le \frac{\kappa}{n+1} + (n-1)D_{n-1} + C\mu_{\norm{\cdot}}\delta.
\end{equation}
Summing this telescoping recursion from $1$ through $n$, we get
\begin{equation}\label{eq:banach-displacement-estimate}
D_n \le \frac{C\kappa(1+\log n)}{n} + C\mu_{\norm{\cdot}}\delta.
\end{equation}
Since $p_n^{-2} = O(n^2)$, by \eqref{eq:banach-displacement-estimate},
\begin{equation}
\begin{aligned}
\sum_{n=1}^N \frac{\E\bigl[\norm{x_n - x_{n-1}}^2\bigr]}{p_n^2} \le C\sum_{n=1}^N n^2\left(\frac{\kappa^2(1+\log n)^2}{n^2} + \mu_{\norm{\cdot}}^2\delta^2\right) = \tilde O\bigl(\kappa^2N + \mu_{\norm{\cdot}}^2\delta^2 N^3\bigr).
\end{aligned}
\end{equation}
The variance-control part of the difference-block cost is therefore
\begin{equation}\label{eq:banach-var-cost}
\tilde O\left(\frac{C_{\mathrm V} L_T^2}{\delta^2} \bigl(\kappa^2 N + \mu_{\norm{\cdot}}^2\delta^2 N^3\bigr)\right).
\end{equation}
Similarly, using \eqref{eq:banach-displacement-estimate} and $p_n^{-3/2} = O(n^{3/2})$,
\begin{equation}
\sum_{n=1}^N \frac{\E\bigl[\norm{x_n-x_{n-1}}\bigr]}{p_n^{3/2}} = \tilde O\bigl(\kappa N^{3/2} + \mu_{\norm{\cdot}}\delta N^{5/2}\bigr).
\end{equation}
Thus the bias-control part of the difference-block cost is
\begin{equation}\label{eq:banach-bias-cost}
\tilde O\left(\frac{C_{\mathrm B} L_T}{\delta}\bigl(\kappa N^{3/2} + \mu_{\norm{\cdot}}\delta N^{5/2}\bigr)\right).
\end{equation}
Finally, substituting
\begin{equation}
N = \tilde \Theta(\kappa/\epsilon),\qquad \delta = \Theta(\epsilon/\mu_{\norm{\cdot}}),
\end{equation}
the refresh cost \eqref{eq:banach-refresh-cost} becomes
\begin{equation}
\tilde O\left(\frac{C_{\mathrm V} \mu_{\norm{\cdot}}^2 \sigma_T^2\kappa}{\epsilon^3}\right),
\end{equation}
the variance-control cost \eqref{eq:banach-var-cost} becomes
\begin{equation}
\tilde O\left(\frac{C_{\mathrm V}\mu_{\norm{\cdot}}^2 L_T^2 \kappa^3}{\epsilon^3} + \frac{C_{\mathrm V} L_T^2 \kappa^3}{\epsilon^3}\right),
\end{equation}
the first term in the bias-control cost \eqref{eq:banach-bias-cost} becomes
\begin{equation}
\tilde O\left(\frac{\mu_{\norm{\cdot}} C_{\mathrm B} L_T \kappa^{5/2}}{\epsilon^{5/2}}\right),
\end{equation}
and the second term becomes
\begin{equation}
\tilde O\left(\frac{\mu_{\norm{\cdot}} C_{\mathrm B} L_T \kappa^{5/2}}{\epsilon^{5/2}}\right).
\end{equation}
Combining all of the contributions yields the stated cost \eqref{eq:banach-page-complexity}.
\end{proof}

\section{Proof of the main high-probability result}

\subsection{Proof of Theorem~\ref{thm:hp-page}}
\begin{proof}[Proof of \cref{thm:hp-page}]
The recursion is an inexact Halpern recursion with
\begin{equation}
U_n = e_{n-1} = v_{n-1} - Tx_{n-1}.
\end{equation}
By \cref{lem:hp-page-schedule}, with probability at least $1 - \eta$, we have
\begin{equation}
\norm{U_N}_{\mathrm{s}} \le \delta
\end{equation}
and
\begin{equation}
\sum_{n=1}^N n\norm{\beta_nU_n - \beta_{n-1}U_{n-1}}_{\mathrm{s}} \le C\delta N(1 + \log(1/\eta)).
\end{equation}
By \cref{cond:hp-page}, we may convert both bounds to the working norm with the factor $\alpha_{\mathrm{s}}$. Applying \cref{lem:pathwise-displacement} and \cref{cond:hp-stability} proves \eqref{eq:hp-residual-bound}. It remains to count the samples used on the same event. We let $d_n := x_n - x_{n-1}$. The proof of \cref{lem:pathwise-displacement} also shows that, for every $n \le N$,
\begin{equation}
n\norm{d_n} \le K(1 + \log(N + 1)) + C\alpha_{\mathrm{s}}\delta N(1 + \log(1/\eta)).
\end{equation}
Under the stated choices of $\delta$ and $N$, the right-hand side is $\tilde O(K)$. Therefore,
\begin{equation}\label{eq:hp-displacement-sum}
\sum_{n=1}^N \frac{\norm{x_n - x_{n-1}}^2}{p_n^2} = \tilde O(K^2N), \qquad \sum_{n=1}^N \frac{\norm{x_n - x_{n-1}}}{p_n} = \tilde O(KN).
\end{equation}
By \eqref{eq:hp-refresh-count}, the refresh cost satisfies
\begin{equation}
\sum_{n=1}^N \mathbf 1_{\cR_n}S_{1,n} \le \tilde O\left(\frac{\Lambda_{N,\eta}\sigma_{\mathrm{s}}^2N(1 + \log(1/\eta))}{\delta^2}\right).
\end{equation}
The term proportional to $\delta^{-1}$ in the second inequality of \eqref{eq:hp-refresh-size} contributes at most
\begin{equation}
\tilde O\left(\frac{\sigma_{\mathrm{s}}N(1 + \log(1/\eta))}{\delta}\right).
\end{equation}
The corresponding initialization term in the first inequality of \eqref{eq:hp-refresh-size} is bounded by the same expression when $N \ge 1$. For the difference blocks, \eqref{eq:hp-difference-size} and \eqref{eq:hp-displacement-sum} give
\begin{equation}
\sum_{n=1}^N S_{2,n} \le \tilde O\left(\frac{\Lambda_{N,\eta}L_{\mathrm{s}}^2K^2N}{\delta^2} + \frac{L_{\mathrm{s}}KN}{\delta} + N\right).
\end{equation}
Substituting $N = \tilde O(K/\epsilon)$ and $\delta = \Theta\bigl(\epsilon/(\alpha_{\mathrm{s}}(1 + \log(1/\eta)))\bigr)$ proves \eqref{eq:hp-complexity}. The terms proportional to $\delta^{-1}$ and $N$ are lower order in $\epsilon$.
\end{proof}

\end{document}

%% file: packages.tex
\usepackage[T1]{fontenc}
\usepackage[utf8]{inputenc}
\usepackage{amsmath,amssymb,amsfonts,amsthm,mathtools}
\usepackage{bm,bbm,mathrsfs}

\usepackage{microtype}
\usepackage{xcolor}
\usepackage{graphicx}
\usepackage{booktabs,tabularx,array,makecell,multirow}
\usepackage{enumitem}
\usepackage{algorithm,algorithmic}
\usepackage{caption,subcaption}
\usepackage{wrapfig}
\usepackage{adjustbox}
\usepackage{nicefrac}
\usepackage{siunitx}
\usepackage{tikz}
\usetikzlibrary{positioning,arrows.meta,calc,matrix}
\usepackage{pgfplots}
\usepackage{pgfplotstable}
\pgfplotsset{compat=1.18}

\usepackage{hyperref}
\hypersetup{colorlinks=true,linkcolor=blue!50!black,citecolor=blue!50!black,urlcolor=blue!50!black}
\usepackage{aliascnt}
\usepackage[nameinlink,capitalise,noabbrev]{cleveref}
\usepackage{titletoc}

\setlist[itemize]{topsep=3pt,itemsep=2pt,parsep=0pt,leftmargin=1.5em}
\setlist[enumerate]{topsep=3pt,itemsep=2pt,parsep=0pt,leftmargin=1.7em}

\newcolumntype{P}[1]{>{\raggedright\arraybackslash}p{#1}}
\newcolumntype{Y}{>{\raggedright\arraybackslash}X}

\makeatletter
\@ifundefined{theorem}{%
  \theoremstyle{plain}
  \newtheorem{theorem}{Theorem}[section]
}{}

\@ifundefined{proposition}{%
  \newaliascnt{proposition}{theorem}
  \newtheorem{proposition}[proposition]{Proposition}
  \aliascntresetthe{proposition}
}{}
\@ifundefined{lemma}{%
  \newaliascnt{lemma}{theorem}
  \newtheorem{lemma}[lemma]{Lemma}
  \aliascntresetthe{lemma}
}{}
\@ifundefined{corollary}{%
  \newaliascnt{corollary}{theorem}
  
  \aliascntresetthe{corollary}
}{}
\@ifundefined{assumption}{%
  \theoremstyle{definition}
  \newaliascnt{assumption}{theorem}
  \newtheorem{assumption}[assumption]{Assumption}
  \aliascntresetthe{assumption}
}{}
\@ifundefined{definition}{%
  \theoremstyle{definition}
  \newaliascnt{definition}{theorem}
  
  \aliascntresetthe{definition}
}{}
\@ifundefined{example}{%
  \theoremstyle{definition}
  \newaliascnt{example}{theorem}
  
  \aliascntresetthe{example}
}{}
\@ifundefined{remark}{%
  \theoremstyle{remark}
  \newaliascnt{remark}{theorem}
  \newtheorem{remark}[remark]{Remark}
  \aliascntresetthe{remark}
}{}
\@ifundefined{condition}{%
  \theoremstyle{definition}
  \newaliascnt{condition}{theorem}
  \newtheorem{condition}[condition]{Condition}
  \aliascntresetthe{condition}
}{}
\makeatother

\crefname{theorem}{Theorem}{Theorems}
\Crefname{theorem}{Theorem}{Theorems}
\crefname{proposition}{Proposition}{Propositions}
\Crefname{proposition}{Proposition}{Propositions}
\crefname{lemma}{Lemma}{Lemmas}
\Crefname{lemma}{Lemma}{Lemmas}
\crefname{corollary}{Corollary}{Corollaries}
\Crefname{corollary}{Corollary}{Corollaries}
\crefname{assumption}{Assumption}{Assumptions}
\Crefname{assumption}{Assumption}{Assumptions}
\crefname{definition}{Definition}{Definitions}
\Crefname{definition}{Definition}{Definitions}
\crefname{example}{Example}{Examples}
\Crefname{example}{Example}{Examples}
\crefname{remark}{Remark}{Remarks}
\Crefname{remark}{Remark}{Remarks}
\crefname{condition}{Condition}{Conditions}
\Crefname{condition}{Condition}{Conditions}

\providecommand{\R}{\mathbb{R}}
\providecommand{\bR}{\mathbb{R}}

\providecommand{\N}{\mathbb{N}}

\providecommand{\E}{\mathbb{E}}
\providecommand{\bE}{\mathbb{E}}

\providecommand{\PP}{\mathbb{P}}

\providecommand{\cA}{\mathcal{A}}

\providecommand{\cD}{\mathcal{D}}

\providecommand{\cF}{\mathcal{F}}

\providecommand{\cR}{\mathcal{R}}

\providecommand{\cY}{\mathcal{Y}}

\providecommand{\dist}{\operatorname{dist}}

\providecommand{\Fix}{\operatorname{Fix}}

\makeatletter
\@ifundefined{ceil}{}{}
\@ifundefined{floor}{}{}
\@ifundefined{paren}{}{}
\@ifundefined{bracket}{}{}
\@ifundefined{set}{}{}
\@ifundefined{abs}{}{}
\@ifundefined{norm}{\DeclarePairedDelimiter{\norm}{\lVert}{\rVert}}{}
\@ifundefined{ip}{\DeclarePairedDelimiterX{\ip}[2]{\langle}{\rangle}{#1,\,#2}}{}
\@ifundefined{gendivx}{\DeclarePairedDelimiterX{\gendivx}[2]{(}{)}{#1\;\delimsize\|\;#2}}{}
\makeatother

